\documentclass{article}
\usepackage{fullpage}

\usepackage[utf8]{inputenc}
\usepackage{authblk}
\usepackage{amsmath}

\title{\bf Beyond the Ideal: Analyzing the Inexact Muon Update}

\author{Egor Shulgin\thanks{Contacts: \texttt{\{shulgin.yegor, sultan.m.rashed\}@gmail.com, francesco@orabona.com, peter.richtarik@kaust.edu.sa}} \quad Sultan Alrashed \quad Francesco Orabona \quad Peter Richtárik}

\affil{King Abdullah University of Science and Technology (KAUST)\\ KAUST Center of Excellence for Generative AI \\ Thuwal, Saudi Arabia}

\date{}

\usepackage[round]{natbib}

\usepackage{stfloats}

\newcommand{\err}{\delta}

\usepackage{parskip}
\usepackage{hyperref} 
\usepackage{subcaption}
\usepackage{booktabs}
\usepackage{dirtytalk}
\usepackage{makecell}
\usepackage{wrapfig}

\usepackage[hyperpageref]{backref}
\hypersetup{ 
    colorlinks=true,
    citecolor=darkscarlet,
    linkcolor=darkpowderblue,
    filecolor=magenta,      
    urlcolor=yaleblue,
    menucolor=gray
}
\renewcommand*{\backref}[1]{}
\renewcommand*{\backrefalt}[4]{%
   \ifcase #1 %
     \footnotesize{(Not cited.)}%
   \or
     \footnotesize{(Cited on page~#2)}%
   \else
     \footnotesize{(Cited on page~#2)}%
\fi }

\usepackage{multirow}
\usepackage{layout}

\usepackage{amsmath}
\usepackage{amsthm}
\usepackage{amssymb}
\usepackage{amsfonts}
\usepackage{mathtools}
\usepackage{siunitx}
\usepackage{lipsum}

\usepackage{enumitem}

\usepackage{bbm}

\usepackage{url}
\usepackage{color}
\usepackage{graphicx}
\usepackage{verbatim}
\usepackage{subcaption}

\usepackage{cleveref}

\usepackage{natbib}

\usepackage{apptools}
\usepackage{thmtools}
\usepackage{algorithm, algpseudocode}

\newcommand{\norm}[1]{\left\| #1 \right\|}

\newcommand{\inp}[2]{\left\langle#1,#2\right\rangle} 

\newcommand{\del}[1]{}

\newcommand{\R}{\mathbb{R}} 

\newcommand{\eqdef}{:=} 

\newcommand{\Exp}[1]{{\mathbb{E}}\left[#1\right]}

\usepackage{tcolorbox}
\usepackage{pifont}
\definecolor{mydarkgreen}{RGB}{39,130,67}
\definecolor{mydarkred}{RGB}{192,47,25}
\definecolor{mydarkorange}{RGB}{39,130,67}

\newcommand{\algname}[1]{{\small \red \sf #1}}

\newtheorem{assumption}{Assumption}
\newtheorem{lemma}{Lemma}

\newtheorem{theorem}{Theorem}

\newtheorem{corollary}{Corollary}

\theoremstyle{plain}

\theoremstyle{definition}

\usepackage{listings}
\usepackage{xcolor}

\definecolor{codegreen}{rgb}{0,0.6,0}
\definecolor{codegray}{rgb}{0.5,0.5,0.5}
\definecolor{codepurple}{rgb}{0.58,0,0.82}
\definecolor{backcolour}{rgb}{0.95,0.95,0.92}

\lstdefinestyle{mystyle}{
    backgroundcolor=\color{backcolour},   
    commentstyle=\color{codegreen},
    keywordstyle=\color{magenta},
    numberstyle=\tiny\color{codegray},
    stringstyle=\color{codepurple},
    basicstyle=\ttfamily\footnotesize,
    breakatwhitespace=false,         
    breaklines=true,                 
    captionpos=b,                    
    keepspaces=true,                 
    numbers=left,                    
    numbersep=5pt,                  
    showspaces=false,                
    showstringspaces=false,
    showtabs=false,                  
    tabsize=2
}

\definecolor{darkscarlet}{rgb}{0.34, 0.01, 0.1}
\definecolor{yaleblue}{rgb}{0.06, 0.3, 0.57}
\definecolor{darkpowderblue}{rgb}{0.0, 0.2, 0.6}
\definecolor{midnightblue}{rgb}{0.1, 0.1, 0.44}

\renewcommand{\algname}[1]{{\color{midnightblue!70!black}\sf #1}}
\newcommand{\approxalg}[1]{{\ttfamily\color{midnightblue!70!black}#1}}

\begin{document}

\maketitle

\begin{abstract}
The \algname{Muon} optimizer has rapidly emerged as a powerful, geometry-aware alternative to \algname{AdamW}, demonstrating strong performance in large-scale training of neural networks. However, a critical theory-practice disconnect exists: \algname{Muon}'s efficiency relies on fast, approximate orthogonalization, yet all prior theoretical work analyzes an idealized, computationally intractable version assuming exact SVD-based updates. This work moves beyond the ideal by providing the first analysis of the \textit{inexact} orthogonalized update at \algname{Muon}'s core. We develop our analysis within the general framework of Linear Minimization Oracle (LMO)-based optimization, introducing a realistic additive error model to capture the inexactness of practical approximation schemes. Our analysis yields explicit bounds that quantify performance degradation as a function of the LMO inexactness/error, $\err$. We reveal a fundamental coupling between this inexactness and the optimal step size and momentum: lower oracle precision requires a smaller step size but larger momentum parameter. These findings elevate the approximation procedure (e.g., the number of \approxalg{Newton-Schulz} steps) from an implementation detail to a critical parameter that must be \emph{co-tuned} with the learning schedule. NanoGPT experiments directly confirm the predicted coupling, with optimal learning rates clearly shifting as approximation precision changes.
\end{abstract}

\section{Introduction}

For over a decade, the landscape of deep learning optimization has been dominated by adaptive first-order methods, with \algname{AdamW}~\citep{kingma2015adam, loshchilov2019decoupled} serving as the de facto standard for training large and complex neural networks. Its robustness and general effectiveness have powered progress across numerous domains. Yet a new class of geometry-aware optimizers has recently emerged, challenging this paradigm. Among them, \algname{Muon}~\citep{jordan2024muon} has quickly gained prominence as a successor to \algname{AdamW}. By leveraging matrix structure in neural network parameters, \algname{Muon} has demonstrated superior performance and scalability, setting new training speed records for models like nanoGPT~\citep{jordan2024moddednanoGPT} and enabling the efficient training of state-of-the-art Large Language Models (LLMs), such as Kimi from Moonshot AI~\citep{team2025kimi}. Benchmarking studies consistently show that \algname{Muon} can be significantly more computationally efficient than \algname{AdamW}~\citep{liu2025muon, shah2025practical, wen2025fantastic}.

However, a critical disconnect lies at the heart of \algname{Muon}'s success. The practical efficiency of the optimizer is entirely predicated on its use of fast, approximate orthogonalization methods, most notably the \approxalg{Newton-Schulz} iteration~\citep{jordan2024muon}. This iterative matrix-multiplication-based procedure provides a computationally cheap way to approximate the orthogonal factor of a matrix's polar decomposition, avoiding a full Singular Value Decomposition (SVD) which would be prohibitively expensive~\citep{grishina2025accelerating}. Yet, a significant gap exists between this practical implementation and its theoretical understanding. All prior theoretical analyses of \algname{Muon}~\citep[e.g.,][]{li2025noteconvergencemuon, shen2025convergence, chen2025muon} have studied an \emph{idealized}, computationally intractable version of the algorithm. In fact, these studies assume access to an exact orthogonalization oracle that computes the perfect SVD-based update. Hence, the existing theory describes an algorithm that is never actually used in practice, leaving the real-world performance of \algname{Muon} potentially unexplained.

This work moves beyond the ideal to provide the first theoretical analysis of the \textit{inexact} \algname{Muon} update. We situate our analysis within the general framework of Linear Minimization Oracle (LMO)-based optimization~\citep{pethick2025training}, also known as the Frank-Wolfe framework~\citep{frank1956algorithm, hazan2008sparse,clarkson2010coresets,jaggi2013revisiting}. This perspective frames the core \algname{Muon} operation as an LMO call over the unit ball with respect to the spectral norm. To account for the realities of practical computation, we introduce a realistic additive error model. This model is directly motivated by the behavior of practical approximation schemes like \approxalg{Newton-Schulz}, allowing us to capture the inexactness inherent in any efficient implementation. By analyzing the algorithm under this inexact LMO, we bridge the critical gap between theory and practice.

\subsection{Contributions}

Our main contributions to the theory and practice of LMO-based optimizers are:

$\bullet$ \textbf{The first analysis of the practical (inexact) \algname{Muon} update.}
We present the first formal convergence analysis for LMO-based methods with an inexact oracle under a realistic additive error model (Assumption~\ref{assump:inexact_lmo_main}). Our general framework provides the first theoretical guarantees for the \textit{implemented} \algname{Muon} update, moving beyond prior works that exclusively analyzed an idealized, intractable version of the algorithm.

$\bullet$ \textbf{Uncovering a fundamental hyperparameter coupling.}
Our analysis reveals that the LMO inexactness, $\err$, is a critical parameter that alters the optimization dynamics. We derive explicit formulas for the step size ($\gamma^* \propto (1+\err)^{-1/4}$) and momentum ($\alpha^* \propto \sqrt{1+\err}$) in the stochastic setting (Corollary~\ref{cor:stochastic_optimal_rate}), uncovering a crucial coupling: a less precise LMO (larger $\err$) requires a {\em smaller} learning rate but a {\em larger} momentum parameter. This elevates the approximation quality from an implementation detail to a core hyperparameter.

$\bullet$ \textbf{Comprehensive theoretical framework.} We establish convergence rates for both deterministic and stochastic settings, with extensions to the $(L^0, L^1)$-smoothness model~\citep{zhang2020why} and a layer-wise setting~\citep{riabinin2025gluon}. These results tightly generalize prior work, exactly recovering the rates for exact LMOs when the inexactness is set to zero.

$\bullet$ \textbf{Empirical validation.}
We validate our key theoretical predictions with experiments on CIFAR-10 and NanoGPT. We confirm that performance degrades as LMO precision decreases and, crucially, our experiments with NanoGPT (Figure~\ref{fig:nanoGPT_heatmaps}) provide clear empirical evidence for the predicted hyperparameter coupling, showing that the optimal learning rate shifts to a lower value when a less precise LMO is used.

\subsection{Related work}

We survey the key developments toward a theoretical understanding of \algname{Muon}.

The conceptual basis for \algname{Muon}'s orthogonalized update was introduced by \citet{jordan2024muon} and theoretically motivated by \citet{bernstein2024old}, who showed that the preconditioned update of a simplified \algname{Shampoo} optimizer is equivalent to steepest descent under the spectral norm. The first formal convergence analyses for \algname{Muon} sought to connect it to existing theoretical frameworks. Notably, \citet{li2025noteconvergencemuon} provided an initial convergence guarantee by viewing the \algname{Muon} update as a matrix-based generalization of normalized \algname{SGD} with momentum, building upon the analysis of \citet{cutkosky2020momentum}.

Following these initial results, two powerful and unifying frameworks emerged concurrently, providing a more general perspective. \citet{pethick2025training} introduced the \algname{Scion} framework, which situates \algname{Muon} within the broader class of methods based on the Linear Minimization Oracle (LMO), a core component of the Frank-Wolfe algorithm~\citep{frank1956algorithm}. Independently, \citet{kovalev2025muon} developed a non-Euclidean trust-region interpretation, also for arbitrary norms. Both frameworks successfully recover the idealized \algname{Muon} update as a special case when the spectral norm is chosen, and provided the first general convergence guarantees for this class of optimizers. 

An idealized trust-region method, called the ball-proximal (``broximal'') point method (\algname{BPM})  was concurrently developed by \citet{BPM}. \algname{BPM} has remarkable theoretical guarantees for convex optimization---it converges in finitely many steps, and has linear rate independent of the condition number, with the geometric factor improving in each iteration---without requiring differentiability, finite-valuedness, or strong convexity. Moreover, the method provably converges to the global minimizer for a certain class of non-convex problems. A non-Euclidean variant was recently proposed and analyzed by \citet{neBPM}, with the non-Euclidean norm playing the role of a hyper-parameter performing a form of geometric preconditioning.
\algname{Muon} can be seen as an approximate version of non-Euclidean \algname{BPM} for a specific norm choice, where the broximal operator is applied to a stochastic linear approximation of the loss instead of the original loss, with momentum added to the mix as a means of handling stochastic noise.

Subsequent work has focused on making these general frameworks more reflective of practical deep learning scenarios. The \algname{Gluon} framework of \citet{riabinin2025gluon} extended the analysis to a more realistic layer-wise setting and introduced a generalized non-Euclidean smoothness model to better capture the heterogeneous structure of neural networks. Concurrently, \citet{pethick2025generalized} introduced a clipped \algname{Scion} variant that provides guarantees under a similar generalized smoothness condition. Further, \citet{Drop-Muon} have shown that it may be theoretically suboptimal to update {\em all} the layers of a neural network in each iteration, and as a remedy, proposed and analyzed (also under generalized smoothness) the \algname{Drop-Muon} method, showing wall-clock improvements on toy networks. 

An extension of \algname{Muon} to the distributed setting, with support for communication compression, error-feedback and generalized smoothness, was developed by \citet{EF21-Muon}, who proposed the \algname{EF21-Muon} method. The \algname{MuonBP} method of \citet{MuonBP} proposes to apply orthogonalization independently to matrix shards on each device, while periodically performing full orthogonalization to maintain training stability at scale.

Despite all this theoretical understanding, all prior works share a critical limitation: they analyze an idealized algorithm that assumes access to an exact, error-free LMO. This is a significant gap, as the practical efficiency and success of \algname{Muon} are entirely predicated on the use of fast but approximate solvers for the orthogonalized update. The concept of spectral descent and orthogonalized updates has historical roots in deep learning~\citep{carlson2015stochasticaistats, carlson2015stochasticieee, carlson2015preconditioned, tuddenham2022orthogonalising}, but the consequences of its \textit{inexact} computation in modern optimizers have, to our knowledge, never been analyzed. Our work is the first to address this fundamental disconnect between theory and practice.

\section{From Idealized Theory to Practical Implementation}
\label{sec:background}

We begin by formalizing the class of algorithms our analysis covers, starting from the \algname{Muon} optimizer. 

\paragraph{The idealized update}

\algname{Muon}~\citep{jordan2024muon} is matrix-aware, leveraging the geometric structure of two-dimensional weight parameters ($W \in \mathbb{R}^{n \times m}$). Instead of element-wise scaling as in \algname{AdamW}, \algname{Muon} applies momentum and then performs an orthogonalization step on the resulting update matrix. The idealized update has the form
\begin{gather*}
M^k = \alpha M^{k-1} + (1-\alpha) G^k, \quad D^k = T(M^k),\\
 X^{k+1} = X^k - \gamma_k D^k,
\end{gather*}
where $G^k$ is a stochastic gradient, $\gamma_k>0$ is the step size, $M^k$ is the momentum matrix, $\alpha>0$ is the momentum parameter, and $T(\cdot)$ represents the projection onto the set of orthogonal matrices.

As shown by \citet{bernstein2024old}, this update is equivalent to performing steepest descent with respect to the spectral norm geometry, $\|\cdot\| \eqdef \|\cdot\|_\text{sp}$. Given a gradient matrix $G$, the steepest descent direction under this norm is equal to
\begin{equation} \label{eq:steepest_descent}
    {\operatorname{argmin}} \ \left\{\inp{G}{D} \;:\; \|D\|_\text{sp} \le 1 \right\}.
\end{equation}
The solution to \eqref{eq:steepest_descent} is the orthogonal polar factor of the negative gradient, $D = \operatorname{polar}(-G)$. If the SVD of the gradient is $G = U S V^\top$, then the solution is $D = -U V^\top$. This idealized update direction forms the theoretical basis of the \algname{Muon} optimizer. In its practical implementation, the computationally expensive polar decomposition is approximated by the efficient, SVD-free \approxalg{Newton-Schulz} iteration. This practical algorithm demonstrated remarkable empirical success, but was introduced without a formal convergence analysis.

Subsequently, \citet{pethick2025training} (using the {\em Linear Minimization Oracle} (LMO) framework) and \citet{kovalev2025muon} (via a non-Euclidean trust-region framework) provided the first meaningful convergence guarantees for this type of update. These frameworks consider a general update of the form\footnote{From now on, we drop the upper-case matrix notation in favor of simpler, lower-case vector-space notation.} $x^{k+1} = x^k + \gamma_k d^k$, where the direction $d^k$ is the solution to the LMO 
\begin{equation}
    d^k \eqdef {\operatorname{argmin}} \ \left\{\inp{m^k}{d} \;:\; \|d\| \le 1 \right\},
    \label{eq:lmo_definition}
\end{equation}
where $m^k$ is the momentum term, $\inp{\cdot}{\cdot}$ refers to an inner product (trace inner product in matrix spaces), and $\|\cdot\|$ refers to an arbitrary, possibly non-Euclidean norm. These analyses recover the idealized \algname{Muon} update when $\|\cdot\|$ is chosen as the spectral norm.

\paragraph{The practical imperative of approximation}
While these frameworks provide invaluable insight, they analyze an idealized algorithm that is never run in practice. 
For many norms, such as the $\ell_\infty$ norm in \algname{SignSGD}, the LMO in \eqref{eq:lmo_definition} can be computed exactly and with minimal overhead. For the spectral norm, however, the exact LMO solution requires a full, prohibitively expensive, SVD.

This computational bottleneck makes the practical implementations of \algname{Muon} entirely dependent on \textit{approximate updates}. The original optimizer employs 5 steps of the \approxalg{Newton-Schulz} iteration \citep{higham2008functions}, a matrix polynomial-based method that efficiently approximates the polar factor using only fast matrix-matrix multiplications \citep{jordan2024muon}. The importance of this approximation is underscored by an active line of research into developing superior iterative schemes \citep{cesista2025muonoptcoeffs}, such as \approxalg{PolarExpress} \citep{amsel2025polar} and \approxalg{CANS} \citep{grishina2025accelerating}, which offer better error guarantees or faster convergence. This highlights a crucial fact: the algorithm achieving state-of-the-art results is not the idealized one, but one of its many possible inexact instantiations. Yet, all prior theoretical work has analyzed the idealized case, assuming access to an error-free LMO. 

\paragraph{Modeling inexactness}

To bridge this theory-practice divide, we must first establish a realistic model for the error produced by the LMO approximation. We replace the exact direction $d^k$ in the update rule with an inexact direction $\hat{d}^k$, which is assumed to satisfy the following assumption.
\begin{assumption}[\textbf{Inexact LMO}]
    \label{assump:inexact_lmo_main}
    Let $d^k$ be the exact solution to \eqref{eq:lmo_definition}. The inexact solution $\hat{d}^k$ is assumed to satisfy an additive error bound for some $\err_k \ge 0$:
    \begin{equation}
        \|\hat{d}^k - d^k\| \le \err_k.
        \label{eq:inexact_assumption}
    \end{equation}
\end{assumption}

This assumption is not arbitrary; it is directly motivated by the convergence guarantees of the iterative methods used in practice. Algorithms like \approxalg{Newton-Schulz} and \approxalg{PolarExpress} produce an approximation whose error decreases with the number of iterations performed. Specifically, \citet{amsel2025polar} prove that \approxalg{PolarExpress} satisfies an error bound of the form $\|\hat{d}^k - d^k\| \le C|1-l|^{p}$, where $l<1$ and $p$ is determined by the number of iterations (see Appendix~\ref{app:error_model} for details). Since these methods are always run for a small, fixed number of steps in practice (e.g., five steps in the standard \algname{Muon} implementation \citep{jordan2024muon}), their final error is bounded by a constant. Our parameter $\err_k$ models this error, allowing our analysis to capture the behavior of the implemented algorithm. A key feature of our analysis is that we do not assume the output $\hat{d}^k$ to be feasible, i.e., we do not assume $\|\hat{d}^k\| \le 1$, which is a realistic property of these approximation schemes.

\section{Main Theoretical Results}
\label{sec:main_results}

Having established the practical importance of the inexact LMO, we now develop a theoretical framework for analyzing its impact on convergence. Our analysis is general and holds for any norm $\|\cdot\|$. We consider the unconstrained optimization problem
\begin{equation}
    \min \ \left\{f(x) \;:\; x \in \mathcal{X} \right\},
    \label{eq:problem_formulation}
\end{equation}
where $(\mathcal{X}, \|\cdot\|)$ is a normed space. We will use the following standard assumption on the objective function.

\begin{assumption}
\label{assump:smoothness}
The objective function $f: \mathcal{X} \to \R$ is continuously differentiable and its gradient $\nabla f$ is Lipschitz continuous with constant $L \ge 0$ with respect to the norm $\|\cdot\|$. This means for any $x, y \in \mathcal{X}$ we have
\begin{equation*}
    \|\nabla f(x) - \nabla f(y)\|_{\star} \le L \|x-y\|,
\end{equation*}
where $\|\cdot\|_{\star}$ is the dual norm of $\|\cdot\|$. Furthermore, we assume that $f$ is bounded below by $f^* > -\infty$.
\end{assumption}

We begin with the deterministic setting to build intuition, before moving to the full stochastic analysis with momentum.

\subsection{Deterministic case}
\label{sec:deterministic}

In the deterministic setting, we analyze the update
\begin{equation}
    x^{k+1} = x^k + \gamma_k \hat{d}^k,
    \label{eq:deterministic_method}
\end{equation}
where $\gamma_k > 0$ is the step size and $\hat{d}^k$ is the output of an inexact LMO for the gradient $g^k \eqdef \nabla f(x^k)$, i.e.,
\begin{equation*}
    \hat{d}^k \approx \operatorname{argmin} \ \left\{ \inp{g^k}{d} \;:\; \|d\| \le 1\right\}.
\end{equation*}
The inexact direction $\hat{d}^k$ is assumed to satisfy Assumption~\ref{assump:inexact_lmo_main} with some inexactness level $\err_k \ge 0$. We now state our first convergence result for this method.

\begin{theorem}[General Result]
\label{thm:general_deterministic}
Let Assumption~\ref{assump:smoothness} hold. Let the sequence $\{x^k\}_{k=0}^{K-1}$ be generated by the update rule \eqref{eq:deterministic_method} with step sizes $\gamma_k > 0$. Assume the inexact LMO satisfies Assumption~\ref{assump:inexact_lmo_main} with $\err_k < 1$ for all $k$ and let $\Delta^0 \eqdef f(x^0) - f^*$. Then, after $K$ iterations, 
\begin{align*}
  \min \limits_{0 \le k < K} \ \norm{\nabla f(x^k)}_{\star} 
    \le \frac{\Delta^0 + \frac{L}{2} \sum_{k=0}^{K-1} (\gamma_k)^2 (1+\err_k)^2}{\sum_{k=0}^{K-1} \gamma_k (1-\err_k)}.
\end{align*}
\end{theorem}

Theorem~\ref{thm:general_deterministic} provides a general convergence guarantee that explicitly characterizes how the interplay of step sizes $\{\gamma_k\}$ and inexactness levels $\{\err_k\}$ affects convergence. The bound reveals that inexactness degrades performance in two ways: the numerator is amplified by a $(1+\err_k)^2$ factor due to the potential infeasibility of the update, while the denominator, representing total progress, is diminished by a $(1-\err_k)$ factor. Condition $\err_k < 1$ is mathematically necessary to ensure that the denominator is positive, which guarantees that the algorithm makes progress on average. As shown in the proof (see Appendix~\ref{sec:proofs}), this condition is required for the per-iteration descent property and is satisfied by practical approximation schemes like \approxalg{PolarExpress}~\citep{amsel2025polar}. To gain clearer, quantitative insights into these trade-offs, we next analyze the important special case of constant parameters.

\begin{corollary}[Constant Parameters]
\label{cor:constant_deterministic}
Under the conditions of Theorem~\ref{thm:general_deterministic}, if the step size is constant, $\gamma_k = \gamma > 0$, and the LMO error is constant, $\err_k = \err < 1$, for all $k$, then after $K$ iterations we have
\begin{align*}
\frac{1}{K} \sum \limits_{k=0}^{K-1} \norm{\nabla f(x^k)}_{\star} \le \frac{\Delta^0}{K\gamma(1-\err)} + \frac{L \gamma (1+\err)^2}{2(1-\err)}.
\end{align*}
\end{corollary}

Corollary~\ref{cor:constant_deterministic} simplifies the general bound, making the trade-offs more apparent. The bound consists of two terms that exhibit the classical trade-off in the choice of step size $\gamma$: a larger $\gamma$ reduces the first term but increases the second. This structure implies that an optimal convergence rate is achieved when the step size is of the order $\mathcal{O}(1/\sqrt{K})$, which balances these two terms.

The inexactness level $\err$ degrades both terms in the bound. They are each amplified by a factor of $1/(1-\err)$, which arises from the reduced quality of the descent direction. The second term is further penalized by a $(1+\err)^2$ factor, which quantifies the cost of potential infeasibility of the update step.

Our analysis is general, holding for any norm in a framework akin to that of \citet{pethick2025training} and \citet{kovalev2025muon}. For the specific choice of the spectral norm, $\|\cdot\| = \|\cdot\|_{\star}$, this result provides the first convergence guarantee for the practical, inexact \algname{Muon} update. Notably, if we set the inexactness to zero ($\err=0$), our bound recovers the standard $\mathcal{O}(1/\sqrt{K})$ rate for the exact LMO method. Our framework also accommodates more complex step-size rules; in Appendix~\ref{app:adaptive_step_size}, we analyze an adaptive step-size policy and show that it achieves the same optimal rate. To make the trade-off in the constant step size explicit, we now derive the optimal choice of $\gamma$ that minimizes this bound.

By minimizing the right-hand side of the bound in Corollary~\ref{cor:constant_deterministic} with respect to $\gamma$, we obtain the optimal constant step size and the corresponding best rate.

\begin{corollary}[Optimal Parameters]
\label{cor:optimal_deterministic}
Under the conditions of Corollary~\ref{cor:constant_deterministic}, by choosing the optimal constant step size
$\gamma^* = \frac{1}{1+\err} \sqrt{\frac{2\Delta^0}{K L}}$,
the average gradient norm is bounded as
\begin{align*}
    \frac{1}{K} \sum \limits_{k=0}^{K-1} \norm{\nabla f(x^k)}_{\star} \le \frac{1+\err}{1-\err} \sqrt{\frac{2\Delta^0 L}{K}}.
\end{align*}
\end{corollary}

Corollary~\ref{cor:optimal_deterministic} makes the theoretical trade-offs concrete, revealing a direct coupling between the LMO precision and the optimal strategy. \emph{The optimal step size, $\gamma^* \propto 1/(1+\err)$, must decrease as the oracle becomes less precise}. This has a direct practical implication for optimizers like \algname{Muon}: using a less accurate approximation of the orthogonalized update (e.g., by reducing the number of \approxalg{Newton-Schulz} iterations) might require a corresponding decrease in the step size to maintain optimal performance. This adaptation, however, does not fully mitigate the error, as the best achievable rate is degraded by a factor of $\frac{1+\err}{1-\err}$. This factor quantifies the dual cost of the potential update infeasibility (from the $1+\err$ term) and reduced descent quality (from the $1-\err$ term). The explicit dependence of the rate on this degradation factor also provides a theoretical justification for the empirical benefits of more advanced approximation schemes. In fact, methods such as \approxalg{PolarExpress}~\citep{amsel2025polar} and \approxalg{CANS}~\citep{grishina2025accelerating} are designed to achieve a smaller approximation error $\err$ more efficiently, which our analysis shows directly translates to an improved convergence rate for the overall optimization.

Our analysis is a tight generalization of prior work; setting $\err=0$ recovers the exact rate for the idealized LMO method~\citep{kovalev2025muon}. Importantly, while the constant is degraded by the inexactness, the $\mathcal{O}(1/\sqrt{K})$ rate implies an iteration complexity of $\mathcal{O}(1/\varepsilon^2)$ to find an $\varepsilon$-stationary point. This matches the optimal complexity for first-order methods on smooth non-convex problems~\citep{carmon2020lower}. Our result thus establishes that the inexact LMO method remains optimal, while characterizing the degradation as a function of the oracle's error.

\subsection{Stochastic case} 
\label{sec:stochastic}

We now extend our analysis to the more practical setting where the optimizer has access only to a stochastic oracle for the gradient and incorporates momentum.

For this setting, we require two additional standard assumptions. First, we assume access to an unbiased stochastic first-order oracle with bounded variance.
\begin{assumption}
    \label{assump:stochastic_oracle}
    The oracle returns a stochastic gradient $g^k = \nabla f(x^k)+\xi^k$ for a random variable $\xi^k$, such that it is an unbiased estimator of the true gradient, $\Exp{g^k \;|\; x^k} = \nabla f(x^k)$, and has a uniformly bounded variance, $\Exp{\|g^k - \nabla f(x^k)\|_2^2 \;|\; x^k} \le \sigma^2$, for $\sigma^2 \ge 0$.
\end{assumption}
Second, since our analysis is for a general norm $\|\cdot\|$ but the variance is typically assumed to be bounded in the Euclidean norm $\|\cdot\|_2$, we assume a norm compatibility condition that is always true in finite-dimensional spaces. So, we denote by $\rho > 0$ the constant such that for all $v \in \mathcal{X}$, we have $\|v\|_{\star} \le \rho \|v\|_2$.

The method we analyze is a general momentum-based algorithm with an inexact LMO, see Algorithm~\ref{alg:stochastic_inaxact_lmo}.

\begin{algorithm}[t]
\caption{Inexact Generalized \algname{Muon} with Momentum}
\label{alg:stochastic_inaxact_lmo}
\begin{algorithmic}[1]
\State \textbf{Input:} Initial point $x^0$, momentum $m^0$ 
\For{$k = 0, 1, \dots, K-1$}
\State Compute stochastic gradient $g^k$
\State Update momentum: $m^{k+1} = (1-\alpha_k)m^k + \alpha_k g^k$
\State Compute inexact LMO with $m^{k+1}$:
    $\hat{d}^k \approx \underset{\|d\| \le 1}{\operatorname{argmin}} \ \inp{m^{k+1}}{d}$
\State Update parameters: $x^{k+1} = x^k + \gamma_k \hat{d}^k$
\EndFor
\end{algorithmic}
\end{algorithm}

To analyze this algorithm, our proof proceeds by establishing a per-iteration descent guarantee and then bounding the error of the momentum term. While the full proofs are deferred to Appendix~\ref{app:stochastic_proofs}, we present the key lemma for the momentum error here, as it reveals the direct impact of the inexact LMO.

\begin{lemma}
\label{lemma:inexact_momentum_error_main}
Let Assumptions~\ref{assump:inexact_lmo_main}, \ref{assump:smoothness}, \ref{assump:stochastic_oracle} hold. For Algorithm~\ref{alg:stochastic_inaxact_lmo}, the expected momentum error is bounded by
\begin{align*}
\mathbb{E} \|m^{k+1} - \nabla f(x^k)\|_{\star} 
    \le (1-\alpha_k)\mathbb{E} [\|m^k - \nabla f(x^{k-1})\|_{\star}] + \frac{\rho\sigma\sqrt{\alpha_k}}{\sqrt{2-\alpha_k}} + L\gamma_k(1+\err_k).
\end{align*}
\end{lemma}

\textbf{Commentary on the Proof.}
This lemma bounds the expected deviation of the momentum from the true gradient. The critical difference in our analysis, compared to prior work on exact LMOs, arises when bounding the ``gradient drift'' term, which depends on the step length $\|\nabla f(x^{k-1}) - \nabla f(x^k)\|_{\star} \le L\|x^{k-1} - x^k\|$. An exact LMO guarantees a feasible direction $\|d^{k-1}\| \le 1$, yielding a step length of exactly $\gamma_{k-1}$. In our analysis, the potential infeasibility of the inexact direction, $\|\hat{d}^{k-1}\| \le 1 + \err_{k-1}$, leads to a looser step length bound of $\gamma_{k-1}(1+\err_{k-1})$. This modification introduces the crucial $(1+\err_k)$ factor into the final term of the lemma, quantifying the cost of potential infeasibility inherent to any practical LMO approximation. The full proof is provided in Appendix~\ref{app:stochastic_proofs}.

Building on this lemma, we now present the main convergence guarantee for Algorithm~\ref{alg:stochastic_inaxact_lmo}.

\begin{theorem}
\label{thm:stochastic}
Let Assumptions \ref{assump:inexact_lmo_main},\ref{assump:smoothness},\ref{assump:stochastic_oracle} hold, along with the norm compatibility condition. For Algorithm~\ref{alg:stochastic_inaxact_lmo} with parameters $\gamma > 0$, $\alpha \in (0,1)$, and $\err < 1$, after $K$ iterations, the average expected gradient norm is upper bounded by
\begin{align*}
\frac{1}{K}\sum_{k=1}^{K} \mathbb{E}\|\nabla f(x^k)\|_{\star} \le \frac{1}{1-\err}\left[\frac{\Delta^0}{K\gamma} + 2\rho\sigma\left(\frac{1}{\alpha K} + \sqrt{\alpha}\right) + L\gamma\left(\frac{7+3\err}{2} + \frac{2(1+\err)}{\alpha} \right)\right],
\end{align*}
where $\Delta^0 \eqdef f(x^0) - f^*$.
\end{theorem}

Theorem~\ref{thm:stochastic} provides the first convergence guarantee for the practical, inexact LMO-based method with momentum. Compared to the deterministic analysis in Corollary~\ref{cor:constant_deterministic}, the bound contains additional terms dependent on the noise variance $\sigma^2$. With a constant step size, these terms establish convergence that the algorithm will minimize the expected gradient up to a floor governed by the choice of parameters.

Crucially, every term in the bound is amplified by factors involving the inexactness level $\err$, which confirms the intuition that a more precise LMO (smaller $\err$) leads to faster convergence. As a sanity check, setting the inexactness to zero ($\err=0$) allows to recover the corresponding convergence guarantee for the idealized, exact LMO method \citep{kovalev2025muon}, demonstrating that our analysis is a strict generalization.

For clarity, we present this result for the practical case of constant parameters, the case of time-varying parameters is deferred to Appendix~\ref{app:time_varying}. To better understand the complex interplay of the constant parameters, we now derive the optimal choices of $\gamma$ and $\alpha$ that minimize this bound (details in Appendix~\ref{app:stochastic_proofs}).

\begin{corollary}
\label{cor:stochastic_optimal_rate}
Let the conditions of Theorem~\ref{thm:stochastic} hold. By choosing 
$
    \gamma^* = \left(\frac{\Delta^0}{K}\right)^{3/4} \frac{1}{(\sigma^2 L (1+\err))^{1/4}}
$
and
$
    \alpha^* = \sqrt{\frac{\Delta^0 L (1+\err)}{K \sigma^2}}
$,
the average expected gradient norm is bounded as
\begin{align*}
\frac{1}{K}\sum \limits_{k=1}^{K} \mathbb{E} \norm{\nabla f(x^k)}_{\star} = \mathcal{O}\left( \frac{(L\Delta^0)^{1/4}\sigma^{1/2}(1+\err)^{1/4}}{K^{1/4}(1-\err)} \right).
\end{align*}
\end{corollary}

Corollary~\ref{cor:stochastic_optimal_rate} provides the main insights of our stochastic analysis, revealing a fundamental coupling between the LMO precision and the hyperparameter strategy. The best achievable rate is degraded by a factor of $\mathcal{O}\left(\frac{(1+\err)^{1/4}}{1-\err}\right)$. This explicit dependence on $\err$ provides a theoretical justification for the empirical benefits of more advanced approximation schemes like \approxalg{PolarExpress}~\citep{amsel2025polar} and \approxalg{CANS}~\citep{grishina2025accelerating}, as any reduction in approximation error directly improves the convergence rate. Interestingly, the degradation from inexactness is less severe than in the deterministic setting (Corollary~\ref{cor:optimal_deterministic}), where the rate is penalized by $\mathcal{O}\left(\frac{1+\err}{1-\err}\right)$. This suggests that the presence of stochastic noise, which already necessitates a degree of caution, makes the optimization dynamics less sensitive to the additional instability from the inexact LMO.

This rate is achieved when the hyperparameters are adapted to the LMO's precision. The step size, $\gamma^* \propto 1/(1+\err)^{1/4}$, must \textbf{decrease} as the oracle becomes less precise. In contrast, the momentum parameter, $\alpha^* \propto \sqrt{1+\err}$, must \textbf{increase}. This prescribes a clear strategy: a less accurate LMO (larger $\err$) requires more caution (smaller step size) but also greater agility (shorter momentum memory) to adapt to less reliable update directions.

Finally, the $\mathcal{O}(1/K^{1/4})$ rate implies an iteration complexity of $\mathcal{O}(1/\varepsilon^4)$ to find an $\varepsilon$-stationary point. This matches the established lower bounds for first-order stochastic methods on non-convex problems under our assumptions~\citep{ghadimi2013stochastic}, and this complexity is known to be unimprovable in general~\citep{arjevani2023lower}. Our analysis thus establishes that the inexact LMO method remains efficient in a formal sense, while precisely characterizing the degradation as a function of the oracle's error.

\subsection{Extensions and generalizations}

\paragraph{Beyond standard smoothness.}
Our analysis can be extended beyond the standard $L$-smoothness assumption, which can be restrictive for neural networks. In Appendix~\ref{app:l0l1_smooth}, we provide a full analysis for the more general non-Euclidean $(L^0, L^1)$-smoothness model
\citep{zhang2020why, yu2025mirror, riabinin2025gluon}. For this class of functions, we establish an iteration complexity of $$
\mathcal{O}\left(\frac{\Delta^0(1+\err)^2}{(1-\err)^2}\left(\frac{L^0}{\varepsilon^2} + \frac{L^1}{\varepsilon}\right)\right).$$ 
This result recovers the standard $\mathcal{O}(1/\varepsilon^2)$ rate when $L^1=0$. More importantly, it allows for an improved rate over gradient descent in regimes where $L^0$ is small, a phenomenon that has been empirically observed for the training trajectories of nanoGPT on FineWeb~\citep{riabinin2025gluon}. Our analysis shows that the dependence on the inexactness level $\err$ remains consistent, and by setting $\err=0$, our results tightly recover the prior guarantees for the exact LMO method on $(L^0, L^1)$-smooth functions~\citep{vankov2025optimizing}.

\paragraph{Layer-wise analysis.}
In practice, optimizers like \algname{Muon} and \algname{Scion} are applied in a layer-wise manner, often with different norms and update rules for different parameter blocks (e.g., spectral norm for weight matrices, $\ell_\infty$ norm for biases)~\citep{pethick2025training}. This heterogeneity, along with empirical observations of varying smoothness across layers~\citep{riabinin2025gluon}, motivates a more fine-grained analysis. In Appendix~\ref{app:layerwise}, we extend our framework by considering a block-wise structure of the parameter space $x = (x_1, \dots, x_p)$, where each block $x_i$ is associated with its own norm $\|\cdot\|_{(i)}$, smoothness constant $L_i$, and inexactness $\err_i$. 

Our analysis in this setting reveals that the impact of LMO inexactness is not uniform across the network. This suggests that a uniform precision level for all layers may be suboptimal and opens the door to a principled strategy for computational savings, where one could allocate fewer resources (e.g., fewer \approxalg{Newton-Schulz} iterations) to approximate the LMO for layers that are more robust to inexactness.

\section{Experiments} \label{sec:experiments}

We now present empirical results to validate the key predictions of our theoretical analysis. Our experiments are designed to answer two main questions: 
\begin{enumerate}
\item How does optimizer performance degrade as the LMO becomes less precise?
\item Do the hyperparameters for the best performance (step size and momentum) shift in response to the LMO inexactness, as predicted by our analysis?
\end{enumerate}

We use the number of iterations in the approximation algorithm (either \approxalg{Newton-Schulz} or \approxalg{PolarExpress}) as a practical proxy for the inexactness level $\err$, where fewer iterations correspond to a higher $\err$. We conduct experiments on two standard benchmarks: training a nanoGPT model on the FineWeb dataset \citep{penedo2024fineweb} and a CNN on CIFAR-10 \citep{krizhevsky2009learning}. Full details are provided in Appendix~\ref{app:exp_details}.

\subsection{nanoGPT on FineWeb}
\label{sec:exp_nanoGPT}

We train a 124M parameter nanoGPT model with \algname{Muon} with a batch size of 512,000 using the codebase by \citet{jordan2024moddednanoGPT}. For the LMO approximation, we use the \approxalg{PolarExpress} algorithm~\citep{amsel2025polar}.

\begin{figure}[hb]
    \centering
\begin{subfigure}[t]{0.36\linewidth}
    \centering
    \includegraphics[width=\linewidth]{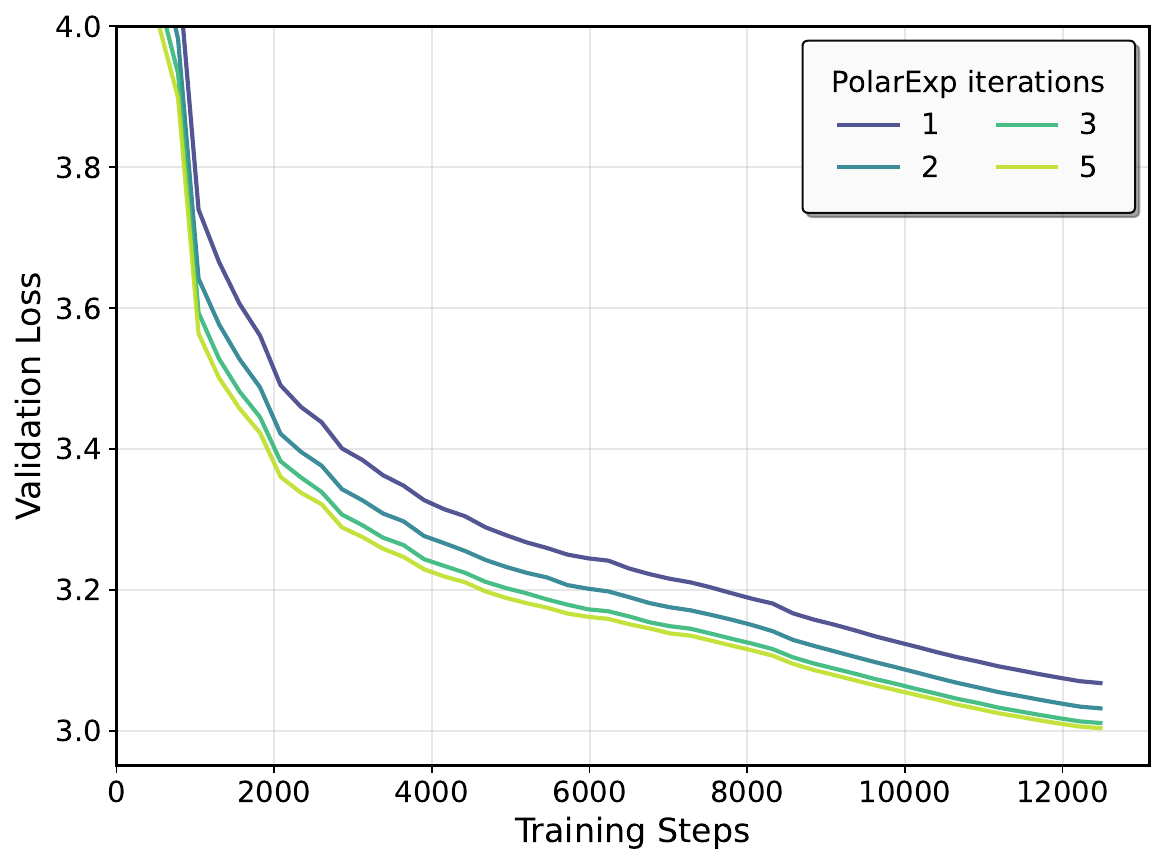}
    \caption{Validation loss behavior; more iterations (smaller $\err$) drive faster convergence.}
    \label{fig:nanoGPT_convergence}
\end{subfigure}
\hfill
\begin{subfigure}[t]{0.62\linewidth}
    \centering
    \includegraphics[width=\linewidth]{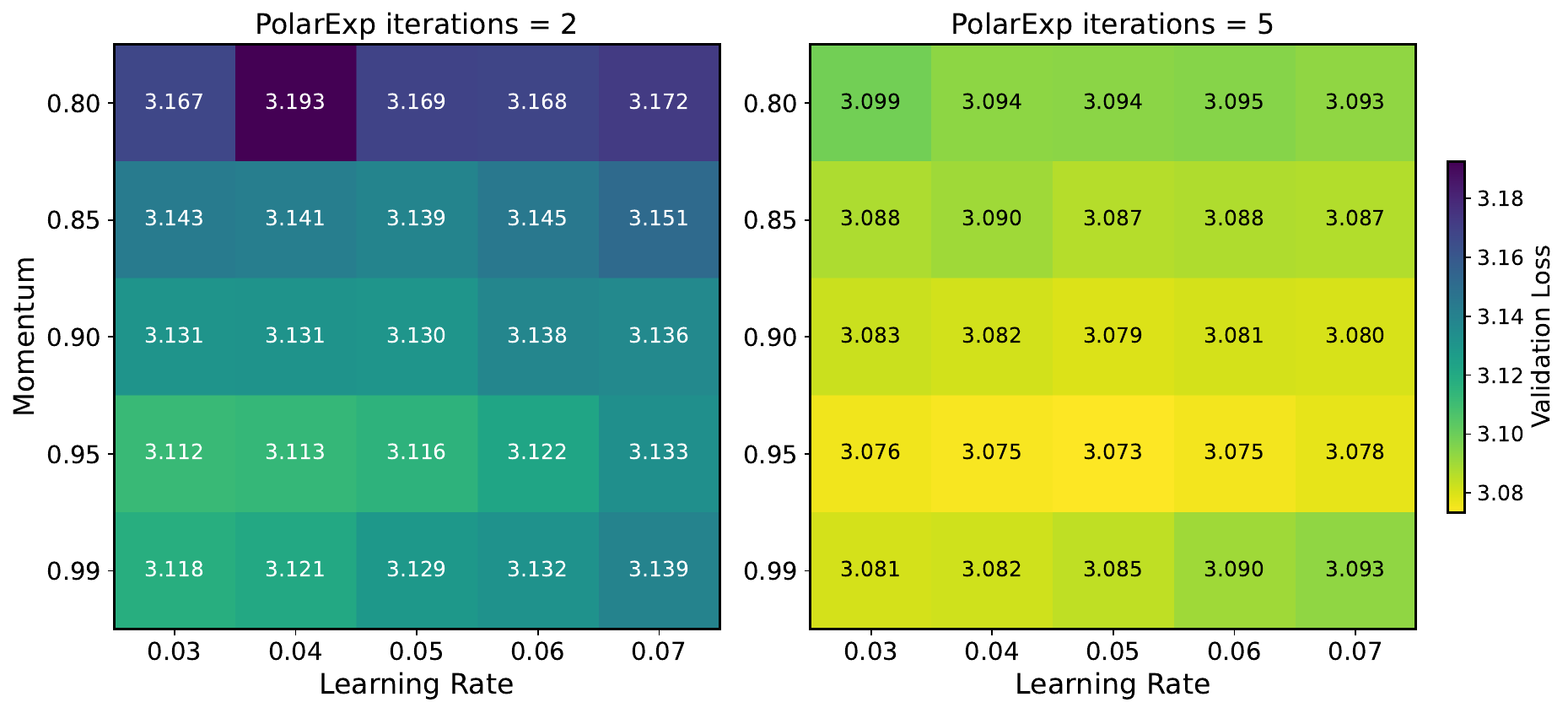}
    \caption{Hyperparameter sweeps for 2 (left) and 5 (right) \approxalg{PolarExpress} iterations; higher precision widens the stable region.}
    \label{fig:nanoGPT_heatmaps}
\end{subfigure}
\caption{NanoGPT on FineWeb: convergence (left) and hyperparameter sensitivity (right) as \approxalg{PolarExpress} (PolarExp) iterations vary.}
\label{fig:nanoGPT_combined}
\end{figure}

\textbf{Validating performance degradation.} Our theory predicts that a less precise LMO (higher $\err$) should lead to degraded convergence performance. Figure~\ref{fig:nanoGPT_convergence} confirms this trend. We trained the model for 12,000 steps and observe that increasing the precision of the LMO by using more \approxalg{PolarExpress} iterations consistently leads to a lower final validation loss. While the most significant gains are seen when moving from one to three iterations, further increases in precision continue to improve performance, though with diminishing returns (see Appendix~\ref{app:additional_nanogpt} for results with up to 8 iterations). This observation directly supports the degradation factor in our convergence bounds.

\textbf{Validating hyperparameter coupling.}
A key prediction of our stochastic analysis (Corollary~\ref{cor:stochastic_optimal_rate}) is the coupling between the LMO inexactness $\err$, the step size $\gamma^*$, and the momentum $\alpha^*$. Figure~\ref{fig:nanoGPT_heatmaps} provides empirical evidence for this theoretical insight. The figure shows the final validation loss after 6,000 training steps across a grid of step sizes and momentum values for two different levels of LMO precision: a highly inexact oracle (2 \approxalg{PolarExp} iterations, top) and a more precise one (5 \approxalg{PolarExp} iterations, bottom).

For the \textbf{highly inexact} case (left panel), the region of best performance (lowest loss) is concentrated at a \textbf{low step size} (around 0.03). When the LMO is made \textbf{more precise} (right panel), the optimal region not only shifts to a \textbf{higher step size} (around 0.05) but also becomes broader, indicating greater stability across different hyperparameter choices. This empirical result aligns with our theory: Corollary~\ref{cor:stochastic_optimal_rate} predicts that as inexactness $\err$ increases (fewer iterations), the step size $\gamma^*$ should decrease.

\subsection{CNN on CIFAR-10}
\label{sec:exp_cifar10}

To test our findings in a different domain, we train a simple CNN on the CIFAR-10 dataset with a batch size of 100. Our implementation is based on the unconstrained \algname{Scion} method~\citep{code_scion}, which we run with a constant step size to isolate the effect of inexactness. For these experiments, we use the standard \approxalg{Newton-Schulz} iteration \citep{jordan2024muon} to approximate the LMO.

\begin{figure}[h]
    \centering
    \begin{subfigure}[t]{0.49\linewidth}
        \centering
        \includegraphics[width=\linewidth]{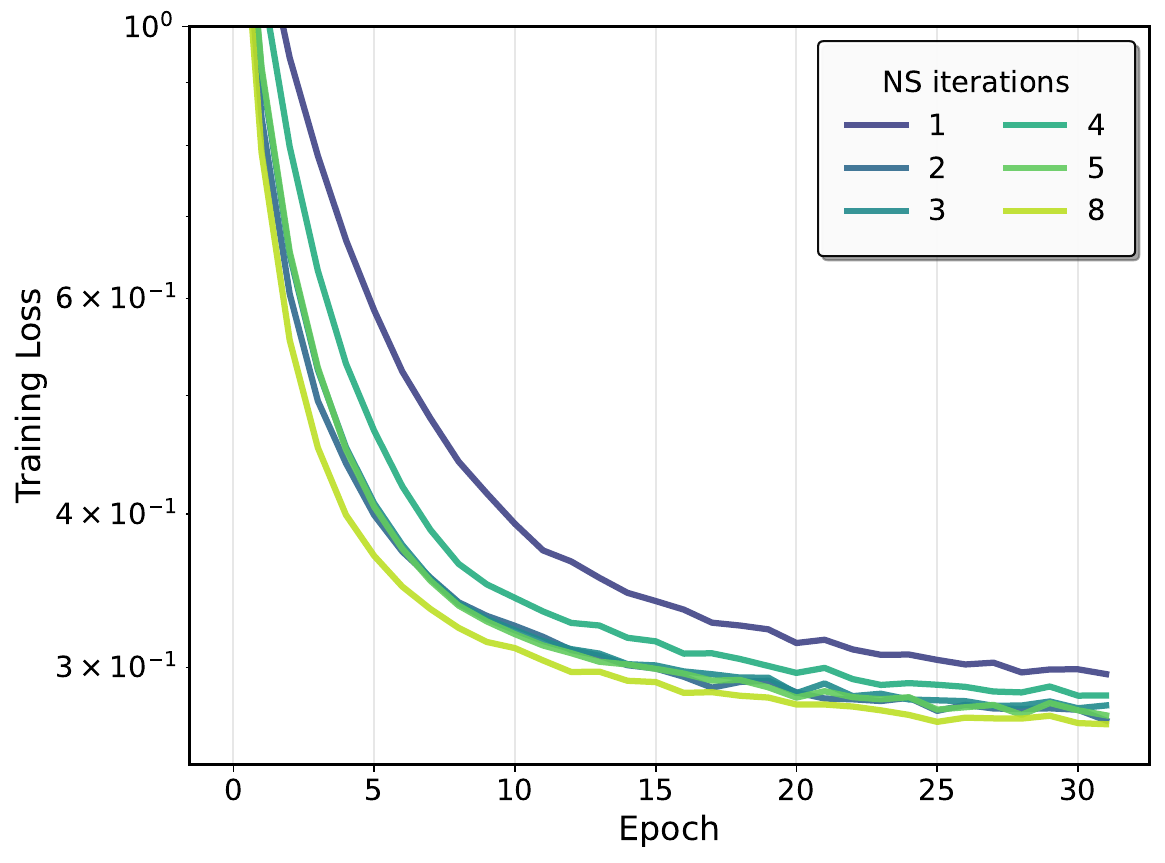}
        \caption{Training loss drops faster as \approxalg{Newton-Schulz} precision increases.}
    \label{fig:cifar_train_loss}
    \end{subfigure}
    \hfill
    \begin{subfigure}[t]{0.48\linewidth}
    \centering
        \includegraphics[width=\linewidth]{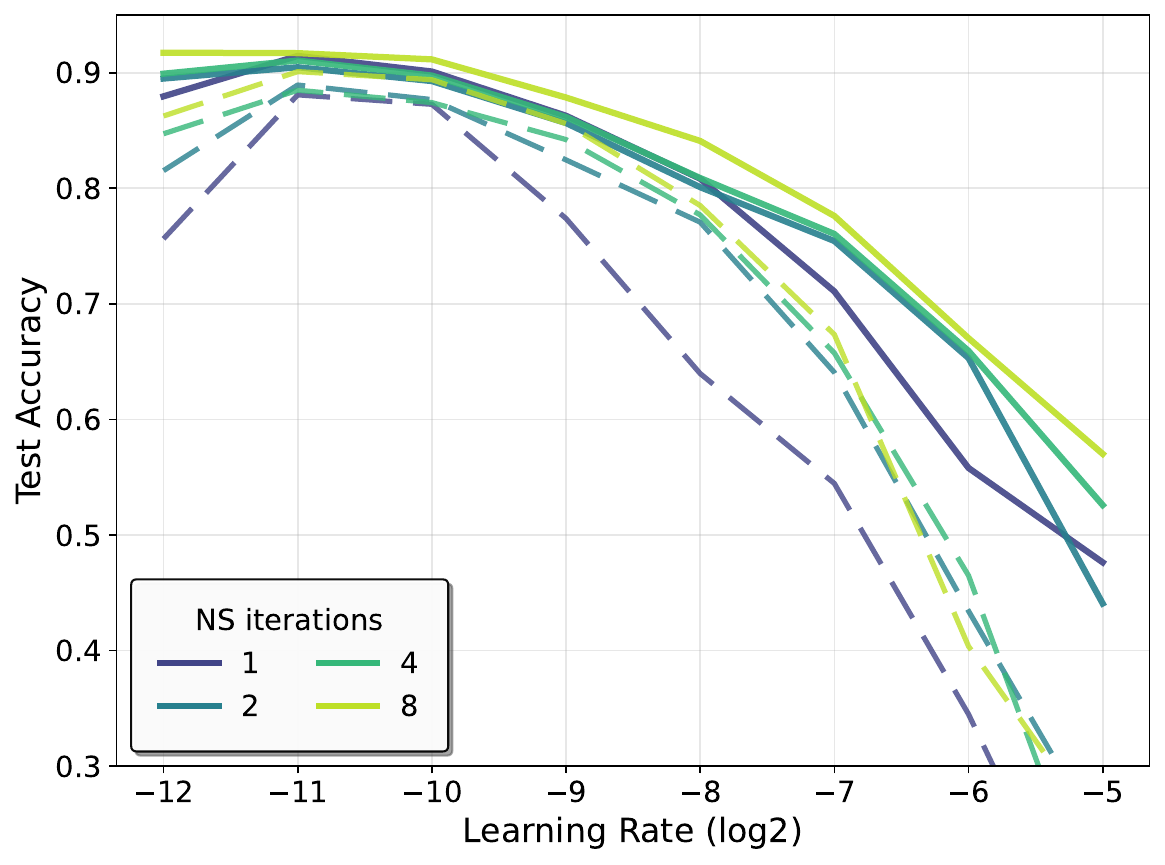}
        \caption{Best (solid) vs. worst (dashed) accuracy across step sizes; higher precision narrows the gap.}
    \label{fig:cifar_variability}
    \end{subfigure}
    \caption{CIFAR-10 training loss and test accuracy across different \approxalg{Newton-Schulz} (NS) precision levels.}
    \label{fig:cifar_combined}
\end{figure}

\textbf{Convergence dynamics.}
Our theory predicts that a less precise LMO should lead to degraded convergence. Figure~\ref{fig:cifar_train_loss} confirms this behavior. We plot the training loss over 32 epochs for different numbers of \approxalg{Newton-Schulz} iterations, using a fixed step size and tuned momentum for each run. The results clearly show that a more precise approximation (e.g., 8 iterations) converges faster and to a lower final training loss compared to less precise approximations (e.g., 1 or 4 iterations). This aligns with the degradation factor present in our theoretical bounds.

\textbf{Performance and stability versus inexactness.}
Beyond the training loss, we investigate how LMO inexactness affects the final test accuracy and the optimizer's stability with respect to its hyperparameters. Figure~\ref{fig:cifar_variability} shows the test accuracy achieved after a sweep over a range of step sizes for different levels of LMO precision. The solid lines represent the best accuracy achieved for a given step size (by tuning momentum), while the dashed lines show the worst.

Two key observations can be made. First, the peak of the solid lines (the best possible performance) consistently increases with the number of \approxalg{Newton-Schulz} iterations, again confirming that higher precision leads to better final performance. Second, the gap between the best (solid) and worst (dashed) performance is large for a highly inexact LMO (1 iteration) and becomes progressively smaller as precision increases. This suggests that a more precise LMO makes the optimizer more robust to the choice of other hyperparameters like momentum, a valuable practical property not explicitly captured by the convergence rate alone.

\section{Conclusion}
\label{sec:conclusion}

This work provides the first formal convergence analysis of LMO-based optimizers with an inexact oracle, addressing the critical gap between the theory and practice of prominent algorithms like \algname{Muon}. By introducing a realistic additive error model, our analysis moves beyond the idealized assumption of a perfect oracle to study the algorithm as it is truly implemented.
Our analysis, for both deterministic and stochastic settings, reveals that there is a coupling between the LMO inexactness, $\err$, and the hyperparameter strategy. Our theory prescribes that a less precise LMO requires a more cautious (smaller) step size, a prediction that our experiments confirm.

\section*{Acknowledgements}

The research reported in this publication was supported by funding from King Abdullah University of Science and Technology (KAUST): i) KAUST Baseline Research Scheme, ii) Center of Excellence for Generative AI, under award number 5940, iii) SDAIA-KAUST Center of Excellence in Artificial Intelligence and Data Science.

\bibliography{ref}

\begin{thebibliography}{44}
\providecommand{\natexlab}[1]{#1}
\providecommand{\url}[1]{\texttt{#1}}
\expandafter\ifx\csname urlstyle\endcsname\relax
  \providecommand{\doi}[1]{doi: #1}\else
  \providecommand{\doi}{doi: \begingroup \urlstyle{rm}\Url}\fi

\bibitem[Amsel et~al.(2025)Amsel, Persson, Musco, and Gower]{amsel2025polar}
Noah Amsel, David Persson, Christopher Musco, and Robert~M Gower.
\newblock The polar express: Optimal matrix sign methods and their application to the {M}uon algorithm.
\newblock \emph{arXiv preprint arXiv:2505.16932}, 2025.
\newblock URL \url{https://arxiv.org/abs/2505.16932}.

\bibitem[Arjevani et~al.(2023)Arjevani, Carmon, Duchi, Foster, Srebro, and Woodworth]{arjevani2023lower}
Yossi Arjevani, Yair Carmon, John~C Duchi, Dylan~J Foster, Nathan Srebro, and Blake Woodworth.
\newblock Lower bounds for non-convex stochastic optimization.
\newblock \emph{Mathematical Programming}, 199\penalty0 (1):\penalty0 165--214, 2023.

\bibitem[Bai et~al.(2025)Bai, Bao, Chen, Chen, Chen, Chen, Chen, Chen, Chen, Chen, Cui, Ding, Dong, Du, Du, Du, Du, Fan, Feng, Fu, Gao, Gao, Gao, Gao, Gu, Guan, Guo, Guo, Hu, Hao, He, He, He, Hong, Hu, Hu, Huang, Huang, Huang, Jiang, Jiang, Jin, Kang, Lai, Li, Li, Li, Li, Li, Li, Li, Li, Li, Lin, Lin, Lin, Liu, Liu, Liu, Liu, Liu, Liu, Liu, Liu, Liu, Liu, Liu, Liu, Liu, Liu, Liu, Lu, Lu, Ma, Ma, Ma, Mao, Mei, Men, Miao, Pan, Peng, Qin, Qu, Shang, Shi, Shi, Song, Su, Su, Sun, Sung, Tang, Tao, Teng, Wang, Wang, Wang, and Wang]{team2025kimi}
Yifan Bai, Yiping Bao, Guanduo Chen, Jiahao Chen, Ningxin Chen, Ruijue Chen, Yanru Chen, Yuankun Chen, Yutian Chen, Zhuofu Chen, Jialei Cui, Hao Ding, Mengnan Dong, Angang Du, Chenzhuang Du, Dikang Du, Yulun Du, Yu~Fan, Yichen Feng, Kelin Fu, Bofei Gao, Hongcheng Gao, Peizhong Gao, Tong Gao, Xinran Gu, Longyu Guan, Haiqing Guo, Jianhang Guo, Hao Hu, Xiaoru Hao, Tianhong He, Weiran He, Wenyang He, Chao Hong, Yangyang Hu, Zhenxing Hu, Weixiao Huang, Zhiqi Huang, Zihao Huang, Tao Jiang, Zhejun Jiang, Xinyi Jin, Yongsheng Kang, Guokun Lai, Cheng Li, Fang Li, Haoyang Li, Ming Li, Wentao Li, Yanhao Li, Yiwei Li, Zhaowei Li, Zheming Li, Hongzhan Lin, Xiaohan Lin, Zongyu Lin, Chengyin Liu, Chenyu Liu, Hongzhang Liu, Jingyuan Liu, Junqi Liu, Liang Liu, Shaowei Liu, T.~Y. Liu, Tianwei Liu, Weizhou Liu, Yangyang Liu, Yibo Liu, Yiping Liu, Yue Liu, Zhengying Liu, Enzhe Lu, Lijun Lu, Shengling Ma, Xinyu Ma, Yingwei Ma, Shaoguang Mao, Jie Mei, Xin Men, Yibo Miao, Siyuan Pan, Yebo Peng, Ruoyu Qin, Bowen Qu, Zeyu Shang,
  Lidong Shi, Shengyuan Shi, Feifan Song, Jianlin Su, Zhengyuan Su, Xinjie Sun, Flood Sung, Heyi Tang, Jiawen Tao, Qifeng Teng, Chensi Wang, Dinglu Wang, Feng Wang, and Haiming Wang.
\newblock Kimi {K2}: Open agentic intelligence.
\newblock \emph{arXiv preprint arXiv:2507.20534}, 2025.
\newblock URL \url{https://arxiv.org/abs/2507.20534}.

\bibitem[Bernstein and Newhouse(2024)]{bernstein2024old}
Jeremy Bernstein and Laker Newhouse.
\newblock Old optimizer, new norm: An anthology.
\newblock In \emph{OPT 2024: Optimization for Machine Learning}, 2024.

\bibitem[Carlson et~al.(2015{\natexlab{a}})Carlson, Cevher, and Carin]{carlson2015stochasticaistats}
David Carlson, Volkan Cevher, and Lawrence Carin.
\newblock Stochastic spectral descent for restricted boltzmann machines.
\newblock In \emph{Artificial Intelligence and Statistics}, pages 111--119. PMLR, 2015{\natexlab{a}}.

\bibitem[Carlson et~al.(2015{\natexlab{b}})Carlson, Hsieh, Collins, Carin, and Cevher]{carlson2015stochasticieee}
David Carlson, Ya-Ping Hsieh, Edo Collins, Lawrence Carin, and Volkan Cevher.
\newblock Stochastic spectral descent for discrete graphical models.
\newblock \emph{IEEE Journal of Selected Topics in Signal Processing}, 10\penalty0 (2):\penalty0 296--311, 2015{\natexlab{b}}.

\bibitem[Carlson et~al.(2015{\natexlab{c}})Carlson, Collins, Hsieh, Carin, and Cevher]{carlson2015preconditioned}
David~E Carlson, Edo Collins, Ya-Ping Hsieh, Lawrence Carin, and Volkan Cevher.
\newblock Preconditioned spectral descent for deep learning.
\newblock \emph{Advances in Neural Information Processing Systems}, 28, 2015{\natexlab{c}}.

\bibitem[Carmon et~al.(2020)Carmon, Duchi, Hinder, and Sidford]{carmon2020lower}
Yair Carmon, John~C Duchi, Oliver Hinder, and Aaron Sidford.
\newblock Lower bounds for finding stationary points i.
\newblock \emph{Mathematical Programming}, 184\penalty0 (1):\penalty0 71--120, 2020.

\bibitem[Cesista et~al.(2025)Cesista, You, and Jordan]{cesista2025muonoptcoeffs}
Franz~Louis Cesista, Jiacheng You, and Keller Jordan.
\newblock {S}queezing 1-2\% efficiency gains out of {M}uon by optimizing the {N}ewton-{S}chulz coefficients, February 2025.
\newblock URL \url{https://leloykun.github.io/ponder/muon-opt-coeffs/}.

\bibitem[Chen et~al.(2025)Chen, Li, and Liu]{chen2025muon}
Lizhang Chen, Jonathan Li, and Qiang Liu.
\newblock {M}uon optimizes under spectral norm constraints.
\newblock \emph{arXiv preprint arXiv:2506.15054}, 2025.
\newblock URL \url{https://arxiv.org/abs/2506.15054}.

\bibitem[Clarkson(2010)]{clarkson2010coresets}
Kenneth~L Clarkson.
\newblock Coresets, sparse greedy approximation, and the {Frank-Wolfe} algorithm.
\newblock \emph{ACM Transactions on Algorithms}, 6\penalty0 (4):\penalty0 1--30, 2010.

\bibitem[Cutkosky and Mehta(2020)]{cutkosky2020momentum}
Ashok Cutkosky and Harsh Mehta.
\newblock Momentum improves normalized {SGD}.
\newblock In \emph{International Conference on Machine Learning}, pages 2260--2268. PMLR, 2020.

\bibitem[Frank and Wolfe(1956)]{frank1956algorithm}
Marguerite Frank and Philip Wolfe.
\newblock An algorithm for quadratic programming.
\newblock \emph{Naval Research Logistics Quarterly}, 3\penalty0 (1-2):\penalty0 95--110, 1956.

\bibitem[Ghadimi and Lan(2013)]{ghadimi2013stochastic}
Saeed Ghadimi and Guanghui Lan.
\newblock Stochastic first- and zeroth-order methods for nonconvex stochastic programming.
\newblock \emph{SIAM Journal on Optimization}, 23\penalty0 (4):\penalty0 2341--2368, 2013.

\bibitem[Grishina et~al.(2025)Grishina, Smirnov, and Rakhuba]{grishina2025accelerating}
Ekaterina Grishina, Matvey Smirnov, and Maxim Rakhuba.
\newblock Accelerating {N}ewton-{S}chulz iteration for orthogonalization via {C}hebyshev-type polynomials.
\newblock \emph{arXiv preprint arXiv:2506.10935}, 2025.
\newblock URL \url{https://arxiv.org/abs/2506.10935}.

\bibitem[Gruntkowska and Richt\'{a}rik(2025)]{neBPM}
Kaja Gruntkowska and Peter Richt\'{a}rik.
\newblock Non-{E}uclidean broximal point method: a blueprint for geometry-aware optimization.
\newblock \emph{arXiv preprint arXiv:2510.00823}, 2025.
\newblock URL \url{https://arxiv.org/abs/2510.00823}.

\bibitem[Gruntkowska et~al.(2025{\natexlab{a}})Gruntkowska, Gaponov, Tovmasyan, and Richt\'{a}rik]{EF21-Muon}
Kaja Gruntkowska, Alexander Gaponov, Zhirayr Tovmasyan, and Peter Richt\'{a}rik.
\newblock Error feedback for {M}uon and friends.
\newblock \emph{arXiv preprint arXiv:2510.00643}, 2025{\natexlab{a}}.
\newblock URL \url{https://arxiv.org/abs/2510.00643}.

\bibitem[Gruntkowska et~al.(2025{\natexlab{b}})Gruntkowska, Li, and Richt\'{a}rik]{BPM}
Kaja Gruntkowska, Hanmin Li, and Aadi Raneand~Peter Richt\'{a}rik.
\newblock The ball-proximal (=``broximal'') point method: a new algorithm, convergence theory, and applications.
\newblock \emph{arXiv preprint arXiv:2502.02002}, 2025{\natexlab{b}}.
\newblock URL \url{https://arxiv.org/abs/2502.02002}.

\bibitem[Gruntkowska et~al.(2025{\natexlab{c}})Gruntkowska, Maziane, Qu, and Richt\'{a}rik]{Drop-Muon}
Kaja Gruntkowska, Yassine Maziane, Zheng Qu, and Peter Richt\'{a}rik.
\newblock Drop-{M}uon: Update less, converge faster.
\newblock \emph{arXiv preprint arXiv:2510.0223}, 2025{\natexlab{c}}.
\newblock URL \url{https://arxiv.org/abs/2510.0223}.

\bibitem[Hazan(2008)]{hazan2008sparse}
Elad Hazan.
\newblock Sparse approximate solutions to semidefinite programs.
\newblock In \emph{Latin American Symposium on Theoretical Informatics}, pages 306--316. Springer, 2008.

\bibitem[Higham(2008)]{higham2008functions}
Nicholas~J Higham.
\newblock \emph{Functions of matrices: Theory and computation}.
\newblock SIAM, 2008.

\bibitem[Jaggi(2013)]{jaggi2013revisiting}
Martin Jaggi.
\newblock Revisiting {F}rank-{W}olfe: Projection-free sparse convex optimization.
\newblock In \emph{International Conference on Machine Learning}, pages 427--435. PMLR, 2013.

\bibitem[Jordan et~al.(2024{\natexlab{a}})Jordan, Bernstein, Rappazzo, Bo\v{z}a, You, Cesista, and Koszarsky]{jordan2024moddednanoGPT}
Keller Jordan, Jeremy Bernstein, Ben Rappazzo, Vlado Bo\v{z}a, Jiacheng You, Franz Cesista, and Braden Koszarsky.
\newblock {M}odded-nano{GPT}: Speedrunning the nano{GPT} baseline, 2024{\natexlab{a}}.
\newblock URL \url{https://github.com/KellerJordan/modded-nanogpt}.
\newblock GitHub repository; additional contributors: @fern-bear.bsky.social, @Grad62304977.

\bibitem[Jordan et~al.(2024{\natexlab{b}})Jordan, Jin, Bo\v{z}a, You, Cesista, Newhouse, and Bernstein]{jordan2024muon}
Keller Jordan, Yuchen Jin, Vlado Bo\v{z}a, Jiacheng You, Franz Cesista, Laker Newhouse, and Jeremy Bernstein.
\newblock {M}uon: An optimizer for hidden layers in neural networks, 2024{\natexlab{b}}.
\newblock URL \url{https://kellerjordan.github.io/posts/muon/}.
\newblock Technical blog post.

\bibitem[Khaled et~al.(2025)Khaled, Ozkara, Yu, Hon, and Par]{MuonBP}
Ahmed Khaled, Kaan Ozkara, Tao Yu, Mingyi Hon, and Youngsuk Par.
\newblock Muon{BP}: Faster {M}uon via block-periodic orthogonalization.
\newblock \emph{arXiv preprint arXiv:2510.16981}, 2025.
\newblock URL \url{https://arxiv.org/abs/2510.16981}.

\bibitem[Kingma and Ba(2015)]{kingma2015adam}
Diederik~P Kingma and Jimmy Ba.
\newblock Adam: A method for stochastic optimization.
\newblock In \emph{International Conference on Learning Representations}, 2015.

\bibitem[Kovalev(2025)]{kovalev2025muon}
Dmitry Kovalev.
\newblock Understanding gradient orthogonalization for deep learning via non-{E}uclidean trust-region optimization, 2025.
\newblock URL \url{https://arxiv.org/abs/2503.12645}.

\bibitem[Krizhevsky(2009)]{krizhevsky2009learning}
Alex Krizhevsky.
\newblock Learning multiple layers of features from tiny images.
\newblock Technical Report Technical Report TR-2009, University of Toronto, 2009.

\bibitem[Li and Hong(2025)]{li2025noteconvergencemuon}
Jiaxiang Li and Mingyi Hong.
\newblock A note on the convergence of {M}uon and further.
\newblock \emph{arXiv preprint arXiv:2502.02900}, 2025.
\newblock URL \url{https://arxiv.org/abs/2502.02900}.

\bibitem[Liu et~al.(2025)Liu, Su, Yao, Jiang, Lai, Du, Qin, Xu, Lu, Yan, Chen, Zheng, Liu, Liu, Yin, He, Zhu, Wang, Wang, Dong, Zhang, Kang, Zhang, Xu, Zhang, Wu, Zhou, and Yang]{liu2025muon}
Jingyuan Liu, Jianlin Su, Xingcheng Yao, Zhejun Jiang, Guokun Lai, Yulun Du, Yidao Qin, Weixin Xu, Enzhe Lu, Junjie Yan, Yanru Chen, Huabin Zheng, Yibo Liu, Shaowei Liu, Bohong Yin, Weiran He, Han Zhu, Yuzhi Wang, Jianzhou Wang, Mengnan Dong, Zheng Zhang, Yongsheng Kang, Hao Zhang, Xinran Xu, Yutao Zhang, Yuxin Wu, Xinyu Zhou, and Zhilin Yang.
\newblock {M}uon is scalable for {LLM} training.
\newblock \emph{arXiv preprint arXiv:2502.16982}, 2025.
\newblock URL \url{https://arxiv.org/abs/2502.16982}.

\bibitem[Loshchilov and Hutter(2019)]{loshchilov2019decoupled}
Ilya Loshchilov and Frank Hutter.
\newblock Decoupled weight decay regularization.
\newblock In \emph{International Conference on Learning Representations}, 2019.

\bibitem[Penedo et~al.(2024)Penedo, Kydl{\'\i}{\v{c}}ek, Lozhkov, Mitchell, Raffel, Von~Werra, and Wolf]{penedo2024fineweb}
Guilherme Penedo, Hynek Kydl{\'\i}{\v{c}}ek, Anton Lozhkov, Margaret Mitchell, Colin~A Raffel, Leandro Von~Werra, and Thomas Wolf.
\newblock The {FineWeb} datasets: Decanting the web for the finest text data at scale.
\newblock \emph{Advances in Neural Information Processing Systems}, 37:\penalty0 30811--30849, 2024.

\bibitem[Pethick et~al.(2025{\natexlab{a}})Pethick, Xie, Antonakopoulos, Zhu, Silveti-Falls, and Cevher]{code_scion}
Thomas Pethick, Wanyun Xie, Kimon Antonakopoulos, Zhenyu Zhu, Antonio Silveti-Falls, and Volkan Cevher.
\newblock Scion, 2025{\natexlab{a}}.
\newblock URL \url{https://github.com/LIONS-EPFL/scion.git}.
\newblock GitHub repository.

\bibitem[Pethick et~al.(2025{\natexlab{b}})Pethick, Xie, Antonakopoulos, Zhu, Silveti-Falls, and Cevher]{pethick2025training}
Thomas Pethick, Wanyun Xie, Kimon Antonakopoulos, Zhenyu Zhu, Antonio Silveti-Falls, and Volkan Cevher.
\newblock Training deep learning models with norm-constrained {LMO}s.
\newblock In \emph{Forty-second International Conference on Machine Learning}, 2025{\natexlab{b}}.

\bibitem[Pethick et~al.(2025{\natexlab{c}})Pethick, Xie, Erdogan, Antonakopoulos, Silveti-Falls, and Cevher]{pethick2025generalized}
Thomas Pethick, Wanyun Xie, Mete Erdogan, Kimon Antonakopoulos, Tony Silveti-Falls, and Volkan Cevher.
\newblock Generalized gradient norm clipping \& non-euclidean $({L}_0, {L}_1)$-smoothness.
\newblock \emph{arXiv preprint arXiv:2506.01913}, 2025{\natexlab{c}}.
\newblock URL \url{https://arxiv.org/abs/2506.01913}.

\bibitem[Riabinin et~al.(2025)Riabinin, Shulgin, Gruntkowska, and Richt{\'a}rik]{riabinin2025gluon}
Artem Riabinin, Egor Shulgin, Kaja Gruntkowska, and Peter Richt{\'a}rik.
\newblock Gluon: Making {M}uon \& {S}cion great again! ({B}ridging theory and practice of {LMO}-based optimizers for {LLMs}).
\newblock \emph{arXiv preprint arXiv:2505.13416}, 2025.
\newblock URL \url{https://arxiv.org/abs/2505.13416}.

\bibitem[Shah et~al.(2025)Shah, Polloreno, Stratos, Monk, Chaluvaraju, Hojel, Ma, Thomas, Tanwer, Shah, Nguyen, Smith, Callahan, Pust, Parmar, Rushton, Mazarakis, Kapila, Srivastava, Singla, Romanski, Vanjani, and Vaswani]{shah2025practical}
Ishaan Shah, Anthony~M. Polloreno, Karl Stratos, Philip Monk, Adarsh Chaluvaraju, Andrew Hojel, Andrew Ma, Anil Thomas, Ashish Tanwer, Darsh~J. Shah, Khoi Nguyen, Kurt Smith, Michael Callahan, Michael Pust, Mohit Parmar, Peter Rushton, Platon Mazarakis, Ritvik Kapila, Saurabh Srivastava, Somanshu Singla, Tim Romanski, Yash Vanjani, and Ashish Vaswani.
\newblock Practical efficiency of {M}uon for pretraining.
\newblock \emph{arXiv preprint arXiv:2505.02222}, 2025.
\newblock URL \url{https://arxiv.org/abs/2505.02222}.

\bibitem[Shen et~al.(2025)Shen, Huang, Huang, Shen, and Zhang]{shen2025convergence}
Wei Shen, Ruichuan Huang, Minhui Huang, Cong Shen, and Jiawei Zhang.
\newblock On the convergence analysis of {M}uon.
\newblock \emph{arXiv preprint arXiv:2505.23737}, 2025.
\newblock URL \url{https://arxiv.org/abs/2505.23737}.

\bibitem[Tuddenham et~al.(2022)Tuddenham, Pr{\"u}gel-Bennett, and Hare]{tuddenham2022orthogonalising}
Mark Tuddenham, Adam Pr{\"u}gel-Bennett, and Jonathan Hare.
\newblock Orthogonalising gradients to speed up neural network optimisation.
\newblock \emph{arXiv preprint arXiv:2202.07052}, 2022.
\newblock URL \url{https://arxiv.org/abs/2202.07052}.

\bibitem[Vankov et~al.(2025)Vankov, Rodomanov, Nedich, Sankar, and Stich]{vankov2025optimizing}
Daniil Vankov, Anton Rodomanov, Angelia Nedich, Lalitha Sankar, and Sebastian~U Stich.
\newblock Optimizing $({L}_0, {L}_1)$-smooth functions by gradient methods.
\newblock In \emph{Proceedings of the International Conference on Learning Representations}, 2025.

\bibitem[Wen et~al.(2025)Wen, Hall, Ma, and Liang]{wen2025fantastic}
Kaiyue Wen, David Hall, Tengyu Ma, and Percy Liang.
\newblock Fantastic pretraining optimizers and where to find them.
\newblock \emph{arXiv preprint arXiv:2509.02046}, 2025.
\newblock URL \url{https://arxiv.org/abs/2509.02046}.

\bibitem[Yang et~al.(2023)Yang, Li, Fatkhullin, and He]{yang2023two}
Junchi Yang, Xiang Li, Ilyas Fatkhullin, and Niao He.
\newblock Two sides of one coin: The limits of untuned {SGD} and the power of adaptive methods.
\newblock \emph{Advances in Neural Information Processing Systems}, 36:\penalty0 74257--74288, 2023.

\bibitem[Yu et~al.(2025)Yu, Jiang, Wan, and Zhang]{yu2025mirror}
Dingzhi Yu, Wei Jiang, Yuanyu Wan, and Lijun Zhang.
\newblock Mirror descent under generalized smoothness.
\newblock \emph{arXiv preprint arXiv:2502.00753}, 2025.
\newblock URL \url{https://arxiv.org/abs/2502.00753}.

\bibitem[Zhang et~al.(2020)Zhang, He, Sra, and Jadbabaie]{zhang2020why}
Jingzhao Zhang, Tianxing He, Suvrit Sra, and Ali Jadbabaie.
\newblock Why gradient clipping accelerates training: A theoretical justification for adaptivity.
\newblock In \emph{International Conference on Learning Representations}, 2020.

\end{thebibliography}

\newpage

\appendix

\tableofcontents

\section{Justification of the Inexact LMO Error Model}
\label{app:error_model}

In this section, we provide additional details on how practical approximation schemes for the orthogonalized update satisfy Assumption~\ref{assump:inexact_lmo_main}.

The \approxalg{PolarExpress} algorithm~\citep{amsel2025polar} is designed to approximate the polar factor $U = \operatorname{polar}(M)$ of a matrix $M$ through iterative refinement. Theorem 4.3 in~\citet{amsel2025polar} establishes that for a matrix $M$ with singular values normalized to lie in $[\ell, 1]$, the \approxalg{PolarExpress} approximation after $r = 2q+1$ (odd) iterations satisfies:
\begin{equation*}
    \|\operatorname{polar}(M) - X_p\|_2 \le |1-\ell^2|^{(q+1)^p},
\end{equation*}
where $\|\cdot\|_2$ denotes the spectral norm (largest singular value). For $r=3$ this yields quadratic convergence, and for $r=5$ cubic convergence. We refer to~\citet{amsel2025polar} for a comprehensive analysis including extensions to rectangular matrices and detailed convergence properties.

The classical \approxalg{Newton-Schulz} iteration~\citep{higham2008functions} exhibits similar convergence properties, with the approximation error decreasing rapidly with the number of iterations.

Since these methods are always run for a small, fixed number of steps in practice (e.g., $p=5$ in \algname{Muon}), their error is bounded by a constant $\err$. Different iteration counts correspond to different values of $\err$, with more iterations producing smaller error. This directly justifies the additive error model in Assumption~\ref{assump:inexact_lmo_main}.

\section{Proofs}
\label{sec:proofs}

\subsection{Deterministic case}

\subsubsection{Proof of Theorem \ref{thm:general_deterministic} (general result)}

The proof establishes a general convergence guarantee for the deterministic inexact LMO method. It proceeds by deriving a per-iteration descent inequality and then summing it over all iterations to obtain a global bound.

\begin{proof}

\textbf{Per-Iteration Descent Inequality.}
We begin with the standard $L$-smoothness inequality for $f$, as stated in Assumption~\ref{assump:smoothness}:
\begin{align*}
    f(x^{k+1}) \le f(x^k) + \inp{\nabla f(x^k)}{x^{k+1}-x^k} + \frac{L}{2}\|x^{k+1}-x^k\|^2.
\end{align*}
Let $g^k \eqdef \nabla f(x^k)$. We substitute the update rule from Equation~\eqref{eq:deterministic_method}, $x^{k+1}-x^k = \gamma_k \hat{d}^k$, into the inequality:
\begin{align} \label{eq:appendix_smoothness_substituted}
    f(x^{k+1}) \le f(x^k) + \gamma_k \inp{g^k}{\hat{d}^k} + \frac{L}{2} \|\gamma_k \hat{d}^k\|^2.
\end{align}
We now bound the two rightmost terms involving the inexact direction $\hat{d}^k$.

\textbf{Bounding the Squared Norm Term.}
Let $d^k$ be the exact LMO solution for the gradient $g^k$. We use the triangle inequality to bound the norm of the inexact direction $\|\hat{d}^k\|$:
\begin{align*}
    \|\hat{d}^k\| = \|\hat{d}^k - d^k + d^k\| \le \|\hat{d}^k - d^k\| + \|d^k\|.
\end{align*}
Using the inexactness from Assumption~\ref{assump:inexact_lmo_main} ($\|\hat{d}^k - d^k\| \le \err_k$) and the property of the exact LMO solution ($\|d^k\| \le 1$), we obtain $\|\hat{d}^k\| \le \err_k + 1$. Therefore, the squared norm term is bounded as:
\begin{align*}
    \|\gamma_k \hat{d}^k\|^2 = (\gamma_k)^2 \|\hat{d}^k\|^2 \le (\gamma_k)^2 (1+\err_k)^2.
\end{align*}

\textbf{Bounding the Inner Product Term.}
We decompose the inner product using the exact direction $d^k$:
\begin{align*}
    \inp{g^k}{\hat{d}^k} = \inp{g^k}{d^k + (\hat{d}^k - d^k)} = \inp{g^k}{d^k} + \inp{g^k}{\hat{d}^k - d^k}.
\end{align*}
By the definition of the LMO and the dual norm, the first term is exactly $- \|g^k\|_{\star}$. The second term is bounded using the definition of dual norm:
\begin{align*}
    \inp{g^k}{\hat{d}^k - d^k} \le \left|\inp{g^k}{\hat{d}^k - d^k}\right| \le \|g^k\|_{\star} \|\hat{d}^k - d^k\| \le \|g^k\|_{\star} \err_k.
\end{align*}
Combining these results gives an upper bound on the inner product:
\begin{align*}
    \inp{g^k}{\hat{d}^k} \le -\|g^k\|_{\star} + \|g^k\|_{\star} \err_k = -\|g^k\|_{\star}(1-\err_k).
\end{align*}

\textbf{Combining the Bounds.}
Substituting these two bounds back into the smoothness inequality \eqref{eq:appendix_smoothness_substituted}, we obtain the per-iteration descent guarantee:
\begin{align*}
    f(x^{k+1}) \le f(x^k) - \gamma_k \|g^k\|_{\star}(1-\err_k) + \frac{L}{2} (\gamma_k)^2 (1+\err_k)^2.
\end{align*}

\textbf{Global Convergence Bound.}
For the descent property to hold, we require $1-\err_k > 0$. Rearranging the per-iteration inequality to isolate the stationarity measure yields:
\begin{align*}
    \gamma_k \|\nabla f(x^k)\|_{\star}(1-\err_k) \le f(x^k) - f(x^{k+1}) + \frac{L}{2} (\gamma_k)^2 (1+\err_k)^2.
\end{align*}
Summing this inequality from $k=0$ to $K-1$:
\begin{align*}
    \sum_{k=0}^{K-1} \gamma_k \|\nabla f(x^k)\|_{\star}(1-\err_k) &\le \sum_{k=0}^{K-1} \left( f(x^k) - f(x^{k+1}) \right) + \frac{L}{2} \sum_{k=0}^{K-1} (\gamma_k)^2 (1+\err_k)^2 \\
    &= f(x^0) - f(x^K) + \frac{L}{2} \sum_{k=0}^{K-1} (\gamma_k)^2 (1+\err_k)^2.
\end{align*}
Using the fact that $f(x^K) \ge f^*$ and letting $\Delta^0 = f(x^0) - f^*$, we have $f(x^0) - f(x^K) \le \Delta^0$. This gives the general bound
\begin{align*}
    \sum_{k=0}^{K-1} \gamma_k \|\nabla f(x^k)\|_{\star}(1-\err_k) \le \Delta^0 + \frac{L}{2} \sum_{k=0}^{K-1} (\gamma_k)^2 (1+\err_k)^2.
\end{align*}
To bound the minimum gradient norm, we note that for any non-negative sequence $\{a_k\}$ and positive sequence $\{w_k\}$, we have
\[
\min_{0 \le j < K} a_j \cdot \sum_{k=0}^{K-1} w_k \le \sum_{k=0}^{K-1} w_k a_k.
\]
Applying this with $a_k = \|\nabla f(x^k)\|_{\star}$ and $w_k = \gamma_k(1-\err_k)$, we get:
\begin{align*}
    \left( \min_{0 \le j < K} \|\nabla f(x^j)\|_{\star} \right) \sum_{k=0}^{K-1} \gamma_k (1-\err_k) \le \sum_{k=0}^{K-1} \gamma_k \|\nabla f(x^k)\|_{\star}(1-\err_k).
\end{align*}
Combining the inequalities and dividing by the sum $\sum \gamma_k(1-\err_k)$ yields the final result stated in Theorem~\ref{thm:general_deterministic}.
\end{proof}

\subsubsection{Proof of Corollary \ref{cor:constant_deterministic} (constant parameters)}
\begin{proof}
The proof follows by simplifying the general bound from Theorem~\ref{thm:general_deterministic} under the specific assumptions of the corollary. We are given that the step size is constant, $\gamma_k = \gamma > 0$, and the LMO error is constant, $\err_k = \err < 1$, for all $k=0, \dots, K-1$.

We start with the summed inequality from the proof of Theorem 1:
\begin{align*}
    \sum_{k=0}^{K-1} \gamma (1-\err) \|\nabla f(x^k)\|_{\star} \le \Delta^0 + \frac{L}{2} \sum_{k=0}^{K-1} \gamma^2 (1+\err)^2.
\end{align*}
The constants can be factored out of the summations:
\begin{align*}
    \gamma (1-\err) \sum_{k=0}^{K-1} \|\nabla f(x^k)\|_{\star} \le \Delta^0 + \frac{L K \gamma^2 (1+\err)^2}{2}.
\end{align*}
To obtain a bound on the average gradient norm, we divide the entire inequality by $K\gamma(1-\err)$:
\begin{align*}
\frac{1}{K} \sum_{k=0}^{K-1} \|\nabla f(x^k)\|_{\star} \le \frac{\Delta^0}{K\gamma(1-\err)} + \frac{L \gamma (1+\err)^2}{2(1-\err)},
\end{align*}
which gives the final result stated in Corollary~\ref{cor:constant_deterministic}.
\end{proof}

\subsubsection{Proof of Corollary \ref{cor:optimal_deterministic} (optimal parameters)}
\begin{proof}
The goal is to find the constant step size $\gamma$ that minimizes the upper bound on the average gradient norm derived in Corollary~\ref{cor:constant_deterministic}. Let the upper bound be denoted by the function $\mathcal{E}(\gamma)$:
\begin{align*}
    \mathcal{E}(\gamma) = \frac{\Delta^0}{K\gamma(1-\err)} + \frac{L \gamma (1+\err)^2}{2(1-\err)}.
\end{align*}
This expression is of the form $A/\gamma + B\gamma$, where $A = \frac{\Delta^0}{K(1-\err)}$ and $B = \frac{L(1+\err)^2}{2(1-\err)}$ are positive constants. This function is convex for $\gamma > 0$. To find the minimizer $\gamma^*$, we take the derivative with respect to $\gamma$ and set it to zero:
\begin{align*}
    \frac{\partial \mathcal{E}}{\partial \gamma} = -\frac{A}{\gamma^2} + B = 0 \implies (\gamma^*)^2 = \frac{A}{B}.
\end{align*}
Substituting the expressions for $A$ and $B$, we obtain
\begin{align*}
    (\gamma^*)^2 = \frac{\frac{\Delta^0}{K(1-\err)}}{\frac{L(1+\err)^2}{2(1-\err)}} = \frac{\Delta^0}{K(1-\err)} \cdot \frac{2(1-\err)}{L(1+\err)^2} = \frac{2\Delta^0}{K L (1+\err)^2}.
\end{align*}
Taking the square root gives the optimal constant step size:
\begin{align*}
    \gamma^* = \sqrt{\frac{2\Delta^0}{K L (1+\err)^2}} = \frac{1}{1+\err} \sqrt{\frac{2\Delta^0}{K L}}.
\end{align*}
To find the best achievable convergence rate, we substitute this optimal step size $\gamma^*$ back into the bound $\mathcal{E}(\gamma)$. At the minimum, the two terms $A/\gamma^*$ and $B\gamma^*$ are equal, so the total error is $2\sqrt{AB}$:
\begin{align*}
    \mathcal{E}(\gamma^*) = 2\sqrt{AB} &= 2 \sqrt{\frac{\Delta^0}{K(1-\err)} \cdot \frac{L(1+\err)^2}{2(1-\err)}} \\
    &= 2 \frac{\sqrt{\Delta^0 L} (1+\err)}{\sqrt{2K}(1-\err)} \\
    &= \frac{\sqrt{2\Delta^0 L}}{ \sqrt{K}} \frac{1+\err}{1-\err}.
\end{align*}
This confirms the final optimized convergence rate stated in Corollary~\ref{cor:optimal_deterministic}.
\end{proof}

\begin{corollary}[Iteration Complexity with Constant Step Size]
Under the conditions of Corollary~\ref{cor:constant_deterministic}, to guarantee that the average stationarity measure satisfies $\frac{1}{K}\sum_{k=0}^{K-1} \|\nabla f(x^k)\|_{\star} \le \varepsilon$ for a target precision $\varepsilon > 0$, it is sufficient to run the algorithm for a number of iterations $K$ of the order of
\begin{align*}
    \mathcal{O}\left( \frac{L\Delta^0}{\varepsilon^2} \cdot \frac{(1+\err)^2}{(1-\err)^2} \right).
\end{align*}
\end{corollary}
\begin{proof}
The result is obtained by setting the optimized bound from Corollary~\ref{cor:optimal_deterministic}, $\frac{1+\err}{1-\err} \sqrt{\frac{2\Delta^0 L}{K}}$, to be less than or equal to $\varepsilon$ and solving for $K$.
\end{proof}

\subsection{Analysis with an adaptive step size}
\label{app:adaptive_step_size}

In the main text, we focus on the analysis with constant or pre-defined diminishing step sizes. Here, we present a complementary result for the deterministic case that considers an adaptive step size, chosen at each iteration to maximize the guaranteed descent. This analysis further highlights the impact of the inexactness level $\err_k$.

\begin{theorem}
\label{thm:optimal_rate_time_varying}
    Let Assumption~\ref{assump:smoothness} hold and let the LMO errors be a sequence $\{\err_k\}$ with $\err_k \in [0, 1)$. If the algorithm is run with the time-varying adaptive step size
    \begin{align*}
        \gamma_k = \frac{\norm{\nabla f(x^k)}_{\star}(1-\err_k)}{L(1+\err_k)^2},
    \end{align*}
    then the minimum gradient norm after $K$ iterations is bounded by:
    \begin{align}
    \label{eq:rate_adaptive_time_varying}
        \min_{0 \leq j < K} \ \norm{\nabla f(x^j)}^2_{\star} \leq \frac{2 L \Delta^0}{\sum_{k=0}^{K-1} \frac{(1-\err_k)^2}{(1+\err_k)^2}}.
    \end{align}
\end{theorem}
\begin{proof}
    The proof starts from the per-iteration descent inequality derived in the proof of Theorem~\ref{thm:general_deterministic}:
    \begin{align*}
    f(x^{k+1}) \le f(x^k) - \gamma_k \norm{\nabla f(x^k)}_{\star}(1-\err_k) + \frac{L (\gamma_k)^2 (1+\err_k)^2}{2}.
    \end{align*}
    The adaptive step size $\gamma_k$ is chosen to minimize the right-hand side of this inequality, which is a quadratic in $\gamma_k$. Its minimizer is precisely the step size given in the theorem statement. Substituting this optimal choice of $\gamma_k$ back into the inequality yields
    \begin{align*}
    f(x^{k+1}) &\leq f(x^k) - \frac{\norm{\nabla f(x^k)}_{\star}^2(1-\err_k)^2}{L(1+\err_k)^2} + \frac{L}{2} \left( \frac{\norm{\nabla f(x^k)}_{\star}(1-\err_k)}{L(1+\err_k)^2} \right)^2 (1+\err_k)^2 \\
    &= f(x^k) - \frac{\norm{\nabla f(x^k)}_{\star}^2(1-\err_k)^2}{2L(1+\err_k)^2}.
    \end{align*}
    Rearranging gives a lower bound on the progress at step $k$:
    \begin{align*}
    \frac{(1-\err_k)^2}{(1+\err_k)^2} \norm{\nabla f(x^k)}_{\star}^2 \leq 2L \left( f(x^k) - f(x^{k+1}) \right).
    \end{align*}
    Summing from $k=0$ to $K-1$, we get
    \begin{align*}
    \sum_{k=0}^{K-1} \frac{(1-\err_k)^2}{(1+\err_k)^2} \norm{\nabla f(x^k)}_{\star}^2 \leq 2L \sum_{k=0}^{K-1} \left( f(x^k) - f(x^{k+1}) \right) \leq 2L \Delta^0.
    \end{align*}
    Let $w_k = \frac{(1-\err_k)^2}{(1+\err_k)^2}$. We have $\sum_{k=0}^{K-1} w_k \norm{\nabla f(x^k)}_{\star}^2 \leq 2L\Delta^0$. Since $\norm{\nabla f(x^j)}_{\star}^2 \geq \min_{0 \leq i < K} \norm{\nabla f(x^i)}_{\star}^2$ for any $j$, we have
    \begin{align*}
    \left( \min_{0 \leq i < K} \ \norm{\nabla f(x^i)}_{\star}^2 \right) \sum_{k=0}^{K-1} w_k \leq \sum_{k=0}^{K-1} w_k \norm{\nabla f(x^k)}_{\star}^2 \leq 2L \Delta^0.
    \end{align*}
    Solving for the minimum squared gradient norm gives the desired inequality:
    \[
        \min_{0 \leq j < K} \ \norm{\nabla f(x^j)}_{\star}^2 \leq \frac{2 L \Delta^0}{\sum_{k=0}^{K-1} \frac{(1-\err_k)^2}{(1+\err_k)^2}}. \qedhere
    \]
\end{proof}

Theorem~\ref{thm:optimal_rate_time_varying} provides a general expression for the convergence rate, which is governed by the growth of the sum $\sum_{k=0}^{K-1} \frac{(1-\err_k)^2}{(1+\err_k)^2}$ in the denominator of the bound. This allows us to analyze the impact of specific schedules for the inexactness level $\{\err_k\}$. The following corollary establishes a simple and practical sufficient condition for achieving the optimal rate.

\begin{corollary}
\label{cor:rate_conditions}
    Under the assumptions of Theorem~\ref{thm:optimal_rate_time_varying}, the optimal convergence rate of $\mathcal{O}(1/\sqrt{K})$ for the minimum gradient norm is achieved if the error sequence $\{\err_k\}$ is uniformly bounded away from 1. Specifically, if there exists a $\err_{\max} = \err \in [0, 1)$ such that $\err_k \le \err_{\max}$ for all $k$, then
    \begin{align}
    \label{eq:rate_bounded_delta}
        \min_{0 \le j < K} \norm{\nabla f(x^j)}_{\star} \le \frac{1+\err}{1-\err} \sqrt{\frac{2 L \Delta^0}{K}}.
    \end{align}
\end{corollary}
\begin{proof}
    If $\err_k \le \err$, then the term $w_k \eqdef \frac{(1-\err_k)^2}{(1+\err_k)^2}$ is bounded below by a positive constant:
    \begin{align*}
        w_k \ge \frac{(1-\err)^2}{(1+\err)^2} > 0.
    \end{align*}
    Therefore, the sum in the denominator of \eqref{eq:rate_adaptive_time_varying} grows at least linearly with $K$:
    \begin{align*}
        \sum_{k=0}^{K-1} w_k \ge \sum_{k=0}^{K-1} \frac{(1-\err)^2}{(1+\err)^2} = K \frac{(1-\err)^2}{(1+\err)^2}.
    \end{align*}
    Substituting this lower bound into the inequality \eqref{eq:rate_adaptive_time_varying} and taking the square root of both sides yields the desired result.
\end{proof}

\begin{corollary}[Iteration Complexity with Adaptive Step Size]
Under the conditions of Corollary~\ref{cor:rate_conditions}, to guarantee that the minimum gradient norm satisfies $\min_{0 \le j < K} \|\nabla f(x^j)\|_{\star} \le \varepsilon$ for a target precision $\varepsilon > 0$, it is sufficient to run the algorithm for a number of iterations $K$ of the order of
\begin{align*}
    \mathcal{O}\left( \frac{L\Delta^0}{\varepsilon^2} \cdot \frac{(1+\err)^2}{(1-\err)^2} \right).
\end{align*}
\end{corollary}
\begin{proof}
The result is obtained by setting the bound from Corollary~\ref{cor:rate_conditions} to be less than or equal to $\varepsilon$ and solving for $K$.
\end{proof}

\subsection{Stochastic case}
\label{app:stochastic_proofs}

This section provides the full proofs for the main theoretical results in the stochastic setting. The analysis proceeds in three main stages. First, in Lemma~\ref{lemma:inexact_descent_appendix}, we establish a crucial per-iteration descent guarantee that holds deterministically for any single step of the algorithm. Second, in Lemma~\ref{lemma:inexact_momentum_error_main}, we derive a bound on the expected error of the momentum term. Finally, in the proof of Theorem~\ref{thm:stochastic}, we combine these results to derive the global convergence guarantee.

\begin{lemma}[Inexact Descent Lemma]
\label{lemma:inexact_descent_appendix}
Let Assumption~\ref{assump:smoothness} hold. For any iteration $k$ of Algorithm~\ref{alg:stochastic_inaxact_lmo}, the following inequality holds:
\begin{align*}
    f(x^{k+1}) \le f(x^k) - \gamma_k \|m^{k+1}\|_{\star}(1-\err_k) + \gamma_k(1+\err_k)\|\nabla f(x^{k+1}) - m^{k+1}\|_{\star} + \frac{3L}{2}(\gamma_k)^2(1+\err_k)^2.
\end{align*}
\end{lemma}
\begin{proof}
The proof establishes a bound on the progress made in a single step. We start with the $L$-smoothness inequality of $f$ from Assumption~\ref{assump:smoothness}:
\begin{align*}
    f(x^{k+1}) \le f(x^k) + \inp{\nabla f(x^k)}{x^{k+1}-x^k} + \frac{L}{2}\|x^{k+1}-x^k\|^2.
\end{align*}
Substituting the update rule $x^{k+1}-x^k = \gamma_k \hat{d}^k$ gives
\begin{align} \label{eq:app_inexact_lemma_start}
    f(x^{k+1}) \le f(x^k) + \gamma_k \inp{\nabla f(x^k)}{\hat{d}^k} + \frac{L}{2}\|\gamma_k \hat{d}^k\|^2.
\end{align}
We decompose the inner product term by introducing the momentum vector $m^{k+1}$:
\begin{align*}
    \gamma_k \inp{\nabla f(x^k)}{\hat{d}^k} = \gamma_k \inp{m^{k+1}}{\hat{d}^k} + \gamma_k \inp{\nabla f(x^k) - m^{k+1}}{\hat{d}^k}.
\end{align*}
The proof proceeds by bounding the key terms involving the inexact direction $\hat{d}^k$.

\textbf{1. Bounding the norm of the inexact direction.}
Let $d^k \eqdef \operatorname{argmin}_{\|d\| \le 1} \inp{m^{k+1}}{d}$ be the exact LMO solution. We use the triangle inequality to bound $\|\hat{d}^k\|$:
\begin{align*}
    \|\hat{d}^k\| = \|\hat{d}^k - d^k + d^k\| \le \|\hat{d}^k - d^k\| + \|d^k\|.
\end{align*}
Using Assumption~\ref{assump:inexact_lmo_main} ($\|\hat{d}^k - d^k\| \le \err_k$) and $\|d^k\| \le 1$, we have $\|\hat{d}^k\| \le \err_k + 1$. This provides a bound for the quadratic term from smoothness:
\begin{align*}
    \|\gamma_k \hat{d}^k\|^2 = (\gamma_k)^2 \|\hat{d}^k\|^2 \le (\gamma_k)^2 (1+\err_k)^2.
\end{align*}

\textbf{2. Bounding the LMO inner product term.}
We explicitly account for the LMO error $\err_k$:
\begin{align*}
    \inp{m^{k+1}}{\hat{d}^k} &= \inp{m^{k+1}}{d^k} + \inp{m^{k+1}}{\hat{d}^k - d^k} \\
    &\le -\|m^{k+1}\|_{\star} + |\inp{m^{k+1}}{\hat{d}^k - d^k}| \\
    &\le -\|m^{k+1}\|_{\star} + \|m^{k+1}\|_{\star} \|\hat{d}^k - d^k\| \quad \text{(by definition of dual norm)} \\
    &\le -\|m^{k+1}\|_{\star}(1-\err_k).
\end{align*}

\textbf{3. Bounding the momentum error term.}
We introduce $\nabla f(x^{k+1})$ to align the error term with the structure needed for subsequent analysis:
\begin{align*}
    \inp{\nabla f(x^k) - m^{k+1}}{\hat{d}^k} &= \inp{\nabla f(x^k) - \nabla f(x^{k+1})}{\hat{d}^k} + \inp{\nabla f(x^{k+1}) - m^{k+1}}{\hat{d}^k} \\
    &\le \|\nabla f(x^k) - \nabla f(x^{k+1})\|_{\star} \|\hat{d}^k\| + \|\nabla f(x^{k+1}) - m^{k+1}\|_{\star} \|\hat{d}^k\|.
\end{align*}
By Assumption~\ref{assump:smoothness}, we have $\|\nabla f(x^k) - \nabla f(x^{k+1})\|_{\star} \le L\|x^k - x^{k+1}\| = L\gamma_k\|\hat{d}^k\|$. Substituting this in yields
\begin{align*}
    \inp{\nabla f(x^k) - m^{k+1}}{\hat{d}^k} &\le (L\gamma_k\|\hat{d}^k\|)\|\hat{d}^k\| + \|\hat{d}^k\| \cdot \|\nabla f(x^{k+1}) - m^{k+1}\|_{\star} \\
    &= L\gamma_k\|\hat{d}^k\|^2 + \|\hat{d}^k\| \cdot \|\nabla f(x^{k+1}) - m^{k+1}\|_{\star}.
\end{align*}
Using our bound $\|\hat{d}^k\| \le 1+\err_k$, this becomes
\begin{align*}
    \inp{\nabla f(x^k) - m^{k+1}}{\hat{d}^k} \le L\gamma_k(1+\err_k)^2 + (1+\err_k)\|\nabla f(x^{k+1}) - m^{k+1}\|_{\star}.
\end{align*}

\textbf{4. Assembling the Final Inequality.}
We now substitute all the bounded components back into inequality \eqref{eq:app_inexact_lemma_start}:
\begin{align*}
f(x^{k+1}) &\le f(x^k) + \gamma_k \left( -\|m^{k+1}\|_{\star}(1-\err_k) \right) \\
&\quad + \gamma_k \left( L\gamma_k(1+\err_k)^2 + (1+\err_k)\|\nabla f(x^{k+1}) - m^{k+1}\|_{\star} \right) \\
&\quad + \frac{L}{2}(\gamma_k)^2(1+\err_k)^2.
\end{align*}
Finally, collecting the terms proportional to $L(\gamma_k)^2(1+\err_k)^2$ gives the result stated in the lemma.
\end{proof}

\subsubsection{Proof of Lemma \ref{lemma:inexact_momentum_error_main}}
\begin{proof}
The proof bounds the expected deviation of the momentum vector $m^{k+1}$ from the true gradient $\nabla f(x^k)$. We remind the reader that the update rule for the momentum is $m^{k+1} = (1-\alpha_k)m^k + \alpha_k g^k$. We start by expanding the error term recursively, obtaining
\begin{align*}
    m^{k+1} - \nabla f(x^k) &= (1-\alpha_k)m^k + \alpha_k g^k - \nabla f(x^k) \\
    &= (1-\alpha_k)(m^k - \nabla f(x^{k-1})) + \alpha_k(g^k - \nabla f(x^k)) + (1-\alpha_k)(\nabla f(x^{k-1}) - \nabla f(x^k)).
\end{align*}
This decomposition separates the error into three distinct components: (1) the decayed error from the previous step, (2) the new stochastic noise, and (3) the gradient drift. By unrolling this recursion from $k$ down to $0$, we can express the total error as a sum of its historical components:
\begin{align*}
    m^{k+1} - \nabla f(x^k) = \left(\prod_{i=0}^k (1-\alpha_i)\right)(m^0 - \nabla f(x^0)) &+ \sum_{i=0}^{k} \left(\prod_{j=i+1}^k (1-\alpha_j)\right) \alpha_i(g^i - \nabla f(x^i)) \\
    &+ \sum_{i=1}^{k} \left(\prod_{j=i}^k (1-\alpha_j)\right)(\nabla f(x^{i-1}) - \nabla f(x^i)).
\end{align*}
We now take the expectation of the norm and apply the triangle inequality, to get
\begin{align*}
    \mathbb{E}[\|m^{k+1} - \nabla f(x^k)\|_{\star}] &\le \left(\prod_{i=0}^k (1-\alpha_i)\right)\mathbb{E}[\|m^0 - \nabla f(x^0)\|_{\star}] \\
    &\quad + \mathbb{E}\left[\left\| \sum_{i=0}^{k} \left(\prod_{j=i+1}^k (1-\alpha_j)\right) \alpha_i(g^i - \nabla f(x^i)) \right\|_{\star}\right] \\
    &\quad + \sum_{i=1}^{k} \left(\prod_{j=i}^k (1-\alpha_j)\right)\mathbb{E}[\|\nabla f(x^{i-1}) - \nabla f(x^i)\|_{\star}].
\end{align*}
To obtain the result in the statement of the lemma, we analyze the bound for constant parameters $\alpha, \gamma, \err$.

\textbf{1. Bounding the Initial Error.}
With the standard initialization $m^0 = g^0$, and using the tower property of expectation along with Assumption~\ref{assump:stochastic_oracle} and the norm compatibility condition, we have
\begin{align*}
    \mathbb{E}[\|m^0 - \nabla f(x^0)\|_{\star}] = \mathbb{E}[\|g^0 - \nabla f(x^0)\|_{\star}] \le \rho\,\mathbb{E}[\|g^0 - \nabla f(x^0)\|_2].
\end{align*}
By Jensen's inequality, $\mathbb{E}[X] \le \sqrt{\mathbb{E}[X^2]}$, this is further bounded by
\begin{align*}
    \rho\sqrt{\mathbb{E}[\|g^0 - \nabla f(x^0)\|_2^2]} \le \rho\sigma.
\end{align*}

\textbf{2. Bounding the Accumulated Drift.}
This is the term modified by our inexact LMO. Using Assumption~\ref{assump:smoothness} ($L$-smoothness) and the fact that $\|x^{i-1}-x^i\| = \gamma\|\hat{d}^{i-1}\| \le \gamma(1+\err)$, we can bound the sum as
\begin{align*}
    \sum_{i=1}^{k} (1-\alpha)^{k-i+1} \mathbb{E}[\|\nabla f(x^{i-1}) - \nabla f(x^i)\|_{\star}] &\le \sum_{i=1}^{k} (1-\alpha)^{k-i+1} L\gamma(1+\err) \\
    &= L\gamma(1+\err) \sum_{j=0}^{k-1} (1-\alpha)^{j+1},
\end{align*}
where we changed the summation index to $j = i-1$. This is a geometric series which we can bound by its infinite sum:
\begin{align*}
    L\gamma(1+\err) \sum_{j=0}^{\infty} (1-\alpha)^{j+1} = L\gamma(1+\err) \frac{1-\alpha}{1-(1-\alpha)} \le \frac{L\gamma(1+\err)}{\alpha}.
\end{align*}

\textbf{3. Bounding the Accumulated Noise.}
Let $V_k = \sum_{i=0}^{k} \alpha(1-\alpha)^{k-i}(g^i - \nabla f(x^i))$. The random vectors $(g^i - \nabla f(x^i))$ are zero-mean and independent conditioned on the past. Using the norm compatibility followed by Jensen's inequality, we have
\begin{align*}
    \mathbb{E}[\|V_k\|_{\star}] \le \rho\,\mathbb{E}[\|V_k\|_2] \le \rho \sqrt{\mathbb{E}[\|V_k\|_2^2]}.
\end{align*}
Since the random vectors are independent and zero-mean, the cross-terms vanish in expectation. So, the expected squared norm of the sum is the sum of the expected squared norms:
\begin{align*}
    \mathbb{E}[\|V_k\|_2^2] &= \sum_{i=0}^{k} \alpha^2(1-\alpha)^{2(k-i)} \mathbb{E}[\|g^i - \nabla f(x^i)\|_2^2] \\
    &\le \sum_{i=0}^{k} \alpha^2(1-\alpha)^{2(k-i)} \sigma^2 = \alpha^2\sigma^2 \sum_{j=0}^{k} ((1-\alpha)^2)^j.
\end{align*}
We bound this geometric series by its infinite sum:
\begin{align*}
    \sum_{j=0}^{k} ((1-\alpha)^2)^j \le \sum_{j=0}^{\infty} ((1-\alpha)^2)^j = \frac{1}{1-(1-\alpha)^2} = \frac{1}{2\alpha - \alpha^2} = \frac{1}{\alpha(2-\alpha)}.
\end{align*}
Substituting this back, we get the bound for the accumulated noise:
\begin{align*}
    \mathbb{E}[\|V_k\|_{\star}] \le \rho \sqrt{\alpha^2\sigma^2 \frac{1}{\alpha(2-\alpha)}} = \rho \sqrt{\frac{\alpha\sigma^2}{2-\alpha}} = \frac{\rho\sigma\sqrt{\alpha}}{\sqrt{2-\alpha}}.
\end{align*}

\textbf{4. Combining the Bounds.}
Combining the bounds for the three components gives the final expression. For the simplified steady-state bound (which holds for any $k$), we combine the infinite-sum bounds for the drift and noise terms with the initial error term, which decays exponentially over time due to the $(1-\alpha)^{k+1}$ factor:
\begin{align*}
    \mathbb{E} [\|m^{k+1} - \nabla f(x^k)\|_{\star}] \le (1-\alpha)^{k+1}\rho\sigma + \frac{\rho\sigma\sqrt{\alpha}}{\sqrt{2-\alpha}} + \frac{L\gamma(1+\err)}{\alpha}.
\end{align*}
This is the result stated for constant parameters in Lemma~\ref{lemma:inexact_momentum_error_main}.
\end{proof}

\subsubsection{Proof of Theorem \ref{thm:stochastic}}
\begin{proof}
The proof derives the final convergence guarantee by combining the per-iteration progress from the \say{Inexact Descent Lemma} (Lemma~\ref{lemma:inexact_descent_appendix}) with the bound on the momentum error from Lemma~\ref{lemma:inexact_momentum_error_main}.

We begin with the per-iteration guarantee from Lemma~\ref{lemma:inexact_descent_appendix}, specialized for constant parameters $\gamma, \alpha, \err$:
\begin{align*}
    f(x^{k+1}) \le f(x^k) - \gamma \|m^{k+1}\|_{\star}(1-\err) + \gamma(1+\err)\|\nabla f(x^{k+1}) - m^{k+1}\|_{\star} + \frac{3L}{2}\gamma^2(1+\err)^2.
\end{align*}
To obtain a guarantee on the true gradient norm, $\|\nabla f(x^{k+1})\|_{\star}$, we apply the reverse triangle inequality to the main descent term:
\begin{align*}
    \|m^{k+1}\|_{\star} = \|\nabla f(x^{k+1}) - (\nabla f(x^{k+1}) - m^{k+1})\|_{\star} \ge \|\nabla f(x^{k+1})\|_{\star} - \|\nabla f(x^{k+1}) - m^{k+1}\|_{\star}.
\end{align*}
Substituting this back into the descent inequality (noting that the leading minus sign flips the inequality for the error term) gives
\begin{align*}
    -\gamma \|m^{k+1}\|_{\star}(1-\err) \le -\gamma \left( \|\nabla f(x^{k+1})\|_{\star} - \|\nabla f(x^{k+1}) - m^{k+1}\|_{\star} \right)(1-\err).
\end{align*}
This yields a new per-iteration guarantee directly in terms of the gradient norm:
\begin{align*}
    f(x^{k+1}) &\le f(x^k) - \gamma\|\nabla f(x^{k+1})\|_{\star}(1-\err) + \gamma(1-\err)\|\nabla f(x^{k+1}) - m^{k+1}\|_{\star} \\
    &\quad + \gamma(1+\err)\|\nabla f(x^{k+1}) - m^{k+1}\|_{\star} + \frac{3L}{2}\gamma^2(1+\err)^2.
\end{align*}
The coefficients for the two momentum error terms sum to $(1-\err) + (1+\err) = 2$. Combining them yields a cleaner expression:
\begin{align*}
    f(x^{k+1}) \le f(x^k) - \gamma\|\nabla f(x^{k+1})\|_{\star}(1-\err) + 2\gamma\|\nabla f(x^{k+1}) - m^{k+1}\|_{\star} + \frac{3L}{2}\gamma^2(1+\err)^2.
\end{align*}
Rearranging to isolate the gradient norm, we get
\begin{align*}
    \gamma(1-\err)\|\nabla f(x^{k+1})\|_{\star} \le f(x^k) - f(x^{k+1}) + 2\gamma\|\nabla f(x^{k+1}) - m^{k+1}\|_{\star} + \frac{3L}{2}\gamma^2(1+\err)^2.
\end{align*}
We now sum this inequality from $k=0$ to $K-1$, take the total expectation, and perform an index shift on the summation of the gradient norm ($j=k+1$), obtaining
\begin{align*}
    \gamma(1-\err)\sum_{j=1}^{K}\mathbb{E}[\|\nabla f(x^j)\|_{\star}] &\le \mathbb{E}\left[\sum_{k=0}^{K-1} (f(x^k) - f(x^{k+1}))\right] 
     + 2\gamma \sum_{k=0}^{K-1} \mathbb{E}[\|\nabla f(x^{k+1}) - m^{k+1}\|_{\star}] \\
     &\quad+ K \frac{3L}{2}\gamma^2(1+\err)^2.
\end{align*}
The first term on the right-hand side is a telescoping sum bounded by $\Delta^0 = f(x^0) - f^*$.

\textbf{Bounding the Sum of Momentum Errors.}
The core of the proof is to bound the sum of momentum errors. We use the triangle inequality to bridge the index gap:
\begin{align*}
    \mathbb{E}[\|\nabla f(x^{k+1}) - m^{k+1}\|_{\star}] \le \mathbb{E}[\|\nabla f(x^{k+1}) - \nabla f(x^k)\|_{\star}] + \mathbb{E}[\|\nabla f(x^k) - m^{k+1}\|_{\star}].
\end{align*}
The first term (gradient drift) is bounded by $L\gamma(1+\err)$. The second term is bounded by Lemma~\ref{lemma:inexact_momentum_error_main}. We use the slightly looser but simpler bound $\frac{\rho\sigma\sqrt{\alpha}}{\sqrt{2-\alpha}} \le \rho\sigma\sqrt{\alpha}$ for $\alpha \in (0,1)$. The full error is thus bounded by
\begin{align*}
    \mathbb{E}[\|\nabla f(x^{k+1}) - m^{k+1}\|_{\star}] \le L\gamma(1+\err) + (1-\alpha)^{k+1}\rho\sigma + \rho\sigma\sqrt{\alpha} + \frac{L\gamma(1+\err)}{\alpha}.
\end{align*}
We now substitute this bound back into the main summation:
\begin{align*}
    \gamma(1-\err)\sum_{j=1}^{K}\mathbb{E}[\|\nabla f(x^j)\|_{\star}] &\le \Delta^0 + K \frac{3L\gamma^2(1+\err)^2}{2} \\
    &\quad + 2\gamma \sum_{k=0}^{K-1} \left( L\gamma(1+\err) + (1-\alpha)^{k+1}\rho\sigma + \rho\sigma\sqrt{\alpha} + \frac{L\gamma(1+\err)}{\alpha} \right).
\end{align*}
Evaluating the summation over $k$ for each component gives
\begin{align*}
    \gamma(1-\err)\sum_{j=1}^{K}\mathbb{E}[\|\nabla f(x^j)\|_{\star}] &\le \Delta^0 + K \frac{3L\gamma^2(1+\err)^2}{2} + 2KL\gamma^2(1+\err) 
     + 2\gamma\left(\sum_{k=0}^{K-1}(1-\alpha)^{k+1}\rho\sigma\right)\\
     &\quad + 2K\gamma\rho\sigma\sqrt{\alpha} + \frac{2KL\gamma^2(1+\err)}{\alpha}.
\end{align*}
For the geometric series, we have $\sum_{k=0}^{K-1}(1-\alpha)^{k+1} \le \sum_{j=1}^{\infty}(1-\alpha)^{j} = \frac{1-\alpha}{\alpha} \le \frac{1}{\alpha}$. Plugging this in gives
\begin{align*}
    \gamma(1-\err)\sum_{j=1}^{K}\mathbb{E}[\|\nabla f(x^j)\|_{\star}] &\le \Delta^0 + K \frac{3L\gamma^2(1+\err)^2}{2} + 2KL\gamma^2(1+\err) \\
    &\quad + \frac{2\gamma\rho\sigma}{\alpha} + 2K\gamma\rho\sigma\sqrt{\alpha} + \frac{2KL\gamma^2(1+\err)}{\alpha}.
\end{align*}
Finally, dividing by $K\gamma(1-\err)$ yields the average bound for the gradient norm. Grouping the terms by their dependency gives the final stated result in Theorem~\ref{thm:stochastic}.
\end{proof}

\subsubsection{Proof of Corollary \ref{cor:stochastic_optimal_rate}}

The proof consists of two parts. First, we derive the asymptotically optimal choices for the step size $\gamma$ and momentum parameter $\alpha$ by minimizing the convergence upper bound from Theorem~\ref{thm:stochastic}. Second, we substitute these optimal parameters back into the bound to obtain the final convergence rate.

\paragraph{1. Optimal Parameter Derivation.}
Our goal is to find $\gamma > 0$ and $\alpha \in (0,1)$ that minimize the expression $\mathcal{E}(\gamma, \alpha)$:
\begin{align*}
    \mathcal{E}(\gamma, \alpha) = \frac{\Delta^0}{K\gamma(1-\err)} + \frac{2\rho\sigma}{\alpha K(1-\err)} + \frac{2\rho\sigma\sqrt{\alpha}}{1-\err} + \frac{L\gamma(1+\err)}{1-\err} \left( \frac{3(1+\err)}{2} + 2 \right) + \frac{2L\gamma(1+\err)}{\alpha(1-\err)}.
\end{align*}
To simplify the algebraic manipulation, we define the following constants:
\begin{itemize}
    \item $C_1 = \frac{\Delta^0}{K(1-\err)}$,
    \item $C_2 = \frac{2\rho\sigma}{K(1-\err)}$,
    \item $C_3 = \frac{2\rho\sigma}{1-\err}$,
    \item $C_4 = \frac{L(1+\err)}{1-\err}\left(\frac{3(1+\err)}{2} + 2\right)$,
    \item $C_5 = \frac{2L(1+\err)}{1-\err}$.
\end{itemize}
Hence, the expression to minimize is $\mathcal{E}(\gamma, \alpha) = \frac{C_1}{\gamma} + \frac{C_2}{\alpha} + C_3\sqrt{\alpha} + C_4\gamma + \frac{C_5\gamma}{\alpha}$. We find the optimal parameters by setting the partial derivatives with respect to $\gamma$ and $\alpha$ to zero, yielding the system of equations
\begin{align}
    \frac{\partial \mathcal{E}}{\partial \gamma} &= -\frac{C_1}{\gamma^2} + C_4 + \frac{C_5}{\alpha} = 0 \label{eq:partial_gamma}, \\
    \frac{\partial \mathcal{E}}{\partial \alpha} &= -\frac{C_2}{\alpha^2} + \frac{C_3}{2\sqrt{\alpha}} - \frac{C_5\gamma}{\alpha^2} = 0.\label{eq:partial_alpha}
\end{align}
An exact closed-form solution is intractable. However, we are interested in the asymptotic behavior as $K \to \infty$, where we expect $\gamma^*, \alpha^* \to 0$. The constants scale as:
$C_1 \propto 1/K$, $C_2 \propto 1/K$, while $C_3, C_4, C_5$ are constants with respect to $K$.

We now solve the system in this asymptotic regime. From \eqref{eq:partial_gamma}, as $\alpha \to 0$, the term $C_5/\alpha$ dominates the constant $C_4$. Thus, the equation asymptotically behaves as
\begin{align}
    \gamma^2 \approx \frac{C_1 \alpha}{C_5}. \label{eq:gamma_asymptotic}
\end{align}
From \eqref{eq:partial_alpha}, since $C_2 \propto 1/K$ and we expect $\gamma$ to decay slower than $1/K$, the term $C_2$ is asymptotically negligible compared to $C_5\gamma$. This gives:
\begin{align}
    \frac{C_3}{2}\alpha^{3/2} \approx C_5\gamma. \label{eq:alpha_asymptotic}
\end{align}
From \eqref{eq:alpha_asymptotic}, we express $\gamma$ in terms of $\alpha$: $\gamma \approx \frac{C_3}{2 C_5} \alpha^{3/2}$. Substituting this into \eqref{eq:gamma_asymptotic}:
\begin{align*}
    \left( \frac{C_3}{2 C_5} \alpha^{3/2} \right)^2 \approx \frac{C_1 \alpha}{C_5} \implies \frac{C_3^2}{4 C_5^2} \alpha^3 \approx \frac{C_1 \alpha}{C_5}.
\end{align*}
Solving for $\alpha^2$ (assuming $\alpha \ne 0$) gives $\alpha^2 \approx \frac{4 C_1 C_5}{C_3^2}$. Now, we substitute back the full definitions of the constants:
\begin{align*}
    (\alpha^*)^2 = \frac{4 \cdot \frac{\Delta^0}{K(1-\err)} \cdot \frac{2L(1+\err)}{1-\err}}{\left(\frac{2\rho\sigma}{1-\err}\right)^2} = \frac{\frac{8\Delta^0 L(1+\err)}{K(1-\err)^2}}{\frac{4(\rho\sigma)^2}{(1-\err)^2}} = \frac{2\Delta^0 L(1+\err)}{K(\rho\sigma)^2}.
\end{align*}
This yields the optimal momentum parameter:
\begin{align*}
    \alpha^* = \frac{\sqrt{2\Delta^0 L(1+\err)}}{\sqrt{K}\rho\sigma}.
\end{align*}
To find the optimal step size $\gamma^*$, we use the relation $\gamma^2 \approx (C_1 \alpha) / C_5$:
\begin{align*}
    (\gamma^*)^2 
    &= \frac{\Delta^0}{K(1-\err)} \cdot \left( \frac{\sqrt{2\Delta^0 L(1+\err)}}{\sqrt{K}\rho\sigma} \right) \cdot \frac{1-\err}{2L(1+\err)} 
    = \frac{\Delta^0 \sqrt{2\Delta^0 L(1+\err)}}{K^{3/2} \rho\sigma \cdot 2L(1+\err)}\\
    &= \frac{(\Delta^0)^{3/2} \sqrt{2L}\sqrt{1+\err}}{K^{3/2} \rho\sigma \cdot 2L(1+\err)} = \frac{(\Delta^0)^{3/2}}{K^{3/2} \sqrt{2L} \rho\sigma\sqrt{1+\err}}.
\end{align*}
Taking the square root gives the optimal step size
\begin{align*}
    \gamma^* = \frac{(\Delta^0)^{3/4}}{2^{1/4} K^{3/4} L^{1/4} (\rho\sigma)^{1/2} (1+\err)^{1/4}}.
\end{align*}

\paragraph{2. Derivation of the Final Convergence Rate.}
We now substitute the asymptotically optimal parameters, $\gamma^*$ and $\alpha^*$, back into the full five-term convergence bound $\mathcal{E}(\gamma, \alpha)$ to determine the best achievable convergence rate. The bound is the sum of five distinct terms:
\begin{align*}
    \mathcal{E}(\gamma^*, \alpha^*) = \underbrace{\frac{C_1}{\gamma^*}}_{\text{Term 1}} + \underbrace{\frac{C_2}{\alpha^*}}_{\text{Term 2}} + \underbrace{C_3\sqrt{\alpha^*}}_{\text{Term 3}} + \underbrace{C_4\gamma^*}_{\text{Term 4}} + \underbrace{\frac{C_5\gamma^*}{\alpha^*}}_{\text{Term 5}}.
\end{align*}
We compute the explicit value of each term to verify the dominant rate and analyze the contribution of the non-dominant, higher-order terms.

\textbf{Dominant Terms ($\mathcal{O}(1/K^{1/4})$).}
These are the terms that dictate the asymptotic convergence rate.

\textit{Term 1: The Initial Gap Component.}
\begin{align*}
    \text{Term 1} = \frac{C_1}{\gamma^*} &= \frac{\Delta^0}{K(1-\err)} \cdot \left( \gamma^* \right)^{-1} = \frac{\Delta^0}{K(1-\err)} \cdot \left( \frac{2^{1/4} K^{3/4} L^{1/4} (\rho\sigma)^{1/2} (1+\err)^{1/4}}{(\Delta^0)^{3/4}} \right) \\
    &= \frac{2^{1/4} (\Delta^0)^{1/4} L^{1/4} (\rho\sigma)^{1/2} (1+\err)^{1/4}}{K^{1/4}(1-\err)}.
\end{align*}

\textit{Term 3: The Steady-State Variance Component.}
\begin{align*}
    \text{Term 3} = C_3\sqrt{\alpha^*} &= \frac{2\rho\sigma}{1-\err} \cdot \left( \frac{\sqrt{2\Delta^0 L(1+\err)}}{\sqrt{K}\rho\sigma} \right)^{1/2} = \frac{2\rho\sigma}{1-\err} \cdot \frac{(2\Delta^0 L(1+\err))^{1/4}}{K^{1/4}(\rho\sigma)^{1/2}} \\
    &= \frac{2 \cdot 2^{1/4} (\rho\sigma)^{1/2} (L\Delta^0(1+\err))^{1/4}}{K^{1/4}(1-\err)} = 2 \cdot (\text{Term 1}).
\end{align*}

\textit{Term 5: The Drift/Interaction Component.}
\begin{align*}
    \text{Term 5} = \frac{C_5\gamma^*}{\alpha^*} &= \frac{2L(1+\err)}{1-\err} \cdot \frac{\gamma^*}{\alpha^*} = \frac{2L(1+\err)}{1-\err} \cdot \frac{C_3}{2C_5} \sqrt{\alpha^*} = \frac{C_3\sqrt{\alpha^*}}{1-\err} \cdot \frac{L(1+\err)}{C_5} \\
    &= \frac{1}{2} C_3\sqrt{\alpha^*} = \text{Term 1}.
\end{align*}
The asymptotic balancing ensures Term 1, Term 3, and Term 5 are all of the same order, $\mathcal{O}(1/K^{1/4})$.

\textbf{Higher-Order Terms.}
The remaining terms decay faster and are asymptotically negligible.

\textit{Term 2: The Decaying Variance Component ($\mathcal{O}(1/K^{1/2})$).}
\begin{align*}
    \text{Term 2} = \frac{C_2}{\alpha^*} = \frac{2\rho\sigma}{K(1-\err)} \cdot \left( \frac{\sqrt{K}\rho\sigma}{\sqrt{2\Delta^0 L(1+\err)}} \right) = \frac{\sqrt{2}(\rho\sigma)^2}{\sqrt{K}(1-\err)\sqrt{\Delta^0 L(1+\err)}}.
\end{align*}

\textit{Term 4: The Step Size Bias Component ($\mathcal{O}(1/K^{3/4})$).}
\begin{align*}
    \text{Term 4} = C_4\gamma^* = \frac{L(1+\err)}{1-\err}\left(\frac{3(1+\err)}{2} + 2\right) \cdot \frac{(\Delta^0)^{3/4}}{2^{1/4} K^{3/4} L^{1/4} (\rho\sigma)^{1/2} (1+\err)^{1/4}} = \frac{(7+3\err)L^{3/4}(\Delta^0)^{3/4}(1+\err)^{3/4}}{2^{5/4}K^{3/4}(1-\err)(\rho\sigma)^{1/2}}.
\end{align*}

\textbf{The Full Bound.}
Combining all terms, the full optimized convergence bound is
\begin{align*}
    \mathcal{E}(\gamma^*, \alpha^*) 
    &= \underbrace{\frac{2^{9/4} (\Delta^0)^{1/4} (\rho\sigma)^{1/2} (L(1+\err))^{1/4}}{K^{1/4}(1-\err)}}_{\text{Dominant Part: } \mathcal{O}(K^{-1/4})} +
    \underbrace{\frac{\sqrt{2}(\rho\sigma)^2}{K^{1/2}(1-\err)\sqrt{\Delta^0 L(1+\err)}}}_{\text{Higher-Order: } \mathcal{O}(K^{-1/2})} +
    \underbrace{\frac{(7+3\err)L^{3/4}(\Delta^0)^{3/4}(1+\err)^{3/4}}{2^{5/4}K^{3/4}(1-\err)(\rho\sigma)^{1/2}}}_{\text{Higher-Order: } \mathcal{O}(K^{-3/4})}.
\end{align*}
This confirms that the convergence rate is $\mathcal{O}(1/K^{1/4})$ with the precise dependence of the dominant part on all problem parameters, matching the result stated in Corollary~\ref{cor:stochastic_optimal_rate}.

\subsection{Convergence with time-varying parameters}
\label{app:time_varying}

The analysis with constant parameters provides crucial insights into the fundamental trade-offs of the algorithm. However, the optimal choices for these constant parameters, $\gamma^*$ and $\alpha^*$, depend on problem-specific constants such as the smoothness $L$ and the initial sub-optimality $\Delta^0$, as well as the total iteration budget $K$. In practice, these quantities are typically unknown, making this ``optimal" tuning infeasible.

To develop a more practical and robust guarantee, we now analyze the algorithm with a pre-defined, \textit{time-varying} schedule for the step size and momentum parameters. This approach is ``parameter-agnostic," meaning the schedules are chosen based only on the iteration counter and do not require prior knowledge of the problem's characteristics. The primary goal of this analysis is to prove that the algorithm minimizes the expected gradient (i.e., $\mathbb{E}[\|\nabla f(x^k)\|_{\star}] \to 0$) for any problem satisfying our general assumptions, without any manual tuning of the learning rates with respect to $L$ or $K$.

The cornerstone of this analysis is a general bound on the expected momentum error that holds for any valid time-varying schedule.

\begin{lemma}[Momentum Error Bound for Time-Varying Parameters]
\label{lemma:aligned_momentum_error}
Let Assumptions \ref{assump:smoothness},\ref{assump:stochastic_oracle}, and the norm compatibility condition hold. Let the sequence $\{x^k\}$ be generated by Algorithm~\ref{alg:stochastic_inaxact_lmo} with LMO errors satisfying Assumption~\ref{assump:inexact_lmo_main}. The expected momentum error at iteration $T$ is bounded by
\begin{align*}
    E_{T-1} = \mathbb{E}[\|m^T - \nabla f(x^{T-1})\|_{\star}] &\le \left( \prod_{t=0}^{T-1} (1-\alpha_t) \right) \rho\sigma + \left( \sum_{t=0}^{T-1} \alpha_t^2 \prod_{\tau=t+1}^{T-1} (1-\alpha_\tau) \right)^{1/2} \rho\sigma \\
    &\quad + L \sum_{t=0}^{T-1} \gamma^t(1+\err_t) \prod_{\tau=t+1}^{T-1} (1-\alpha_\tau).
\end{align*}
\end{lemma}

\begin{proof}
The proof proceeds by unrolling the one-step recursion for the momentum error. Let $e_k \eqdef m^{k+1} - \nabla f(x^k)$, $v_k \eqdef g^k - \nabla f(x^k)$ be the stochastic noise, and $s_{k-1} \eqdef \nabla f(x^{k-1}) - \nabla f(x^k)$ be the gradient drift. The error recursion is
\begin{align*}
    e_k = (1-\alpha_k)e_{k-1} + \alpha_k v_k + (1-\alpha_k)s_{k-1}.
\end{align*}
Unrolling this relationship from $k=T-1$ down to $0$ gives the expression for the total error:
\begin{align*}
    e_{T-1} = \left(\prod_{t=0}^{T-1}(1-\alpha_t)\right) e_{-1} + \sum_{t=0}^{T-1} \left(\prod_{\tau=t+1}^{T-1}(1-\alpha_\tau)\right) \left( \alpha_t v_t + (1-\alpha_t)s_{t-1} \right),
\end{align*}
where we define $s_{-1}=0$ and $e_{-1} = m^0 - \nabla f(x^0)$. For the initialization $m^0=g^0$, we can bound $\mathbb{E}[\|e_{-1}\|_{\star}] \le \rho\sigma$. Taking the expectation of the norm and applying the triangle inequality yields
\begin{align*}
    \mathbb{E}[\|e_{T-1}\|_{\star}] &\le \left(\prod_{t=0}^{T-1} (1-\alpha_t)\right) \rho\sigma + \mathbb{E}\left[\left\| \sum_{t=0}^{T-1} \alpha_t v_t \prod_{\tau=t+1}^{T-1}(1-\alpha_\tau) \right\|_{\star}\right] \\
    &\quad + \mathbb{E}\left[\left\| \sum_{t=0}^{T-1} (1-\alpha_t) s_{t-1} \prod_{\tau=t+1}^{T-1}(1-\alpha_\tau) \right\|_{\star}\right].
\end{align*}
We now bound the two summation terms separately.

\textbf{1. Bounding the Accumulated Noise Term.}
We apply the norm compatibility condition and then Jensen's inequality, to have
\begin{align*}
    \mathbb{E}\left[\left\| \sum_{t=0}^{T-1} \alpha_t v_t \dots \right\|_{\star}\right] \le \rho \left(\mathbb{E}\left[\left\| \sum_{t=0}^{T-1} \alpha_t v_t \prod_{\tau=t+1}^{T-1}(1-\alpha_\tau) \right\|_{2}^2\right]\right)^{1/2}.
\end{align*}
The stochastic noise vectors $v_t$ are zero-mean and independent conditioned on the past. Therefore, the cross-terms in the squared sum vanish in expectation, yielding
\begin{align*}
    \mathbb{E}\left[\left\| \sum \dots \right\|_{2}^2\right] = \sum_{t=0}^{T-1} \alpha_t^2 \left(\prod_{\tau=t+1}^{T-1}(1-\alpha_\tau)\right)^2 \mathbb{E}[\|v_t\|_2^2] \le \sigma^2 \sum_{t=0}^{T-1} \alpha_t^2 \left(\prod_{\tau=t+1}^{T-1}(1-\alpha_\tau)\right)^2.
\end{align*}
Using the upper bound $(1-x)^2 \le (1-x)$ for any $x \in [0,1]$, we have
\begin{align*}
     \mathbb{E}\left[\left\| \sum \dots \right\|_{2}^2\right] \le \sigma^2 \sum_{t=0}^{T-1} \alpha_t^2 \prod_{\tau=t+1}^{T-1}(1-\alpha_\tau).
\end{align*}
Substituting this back yields the final bound for the noise term:
\begin{align*}
    \text{Noise Term} \le \rho\sigma \left(\sum_{t=0}^{T-1} \alpha_t^2 \prod_{\tau=t+1}^{T-1}(1-\alpha_\tau)\right)^{1/2}.
\end{align*}

\textbf{2. Bounding the Accumulated Drift Term.}
We apply the triangle inequality for sums, Assumption~\ref{assump:smoothness}, and the step-size bound $\|x^t - x^{t-1}\| \le \gamma^{t-1}(1+\err_{t-1})$, to obtain
\begin{align*}
    \mathbb{E}\left[\left\| \sum (1-\alpha_t) s_{t-1} \dots \right\|_{\star}\right] &\le \sum_{t=0}^{T-1} (1-\alpha_t) \mathbb{E}[\|s_{t-1}\|_{\star}] \prod_{\tau=t+1}^{T-1}(1-\alpha_\tau) \\
    &\le L \sum_{t=0}^{T-1} (1-\alpha_t) \gamma^{t-1}(1+\err_{t-1}) \prod_{\tau=t+1}^{T-1}(1-\alpha_\tau).
\end{align*}
Using the bound $(1-\alpha_t) \le 1$ and re-indexing the summation gives the simplified form
\begin{align*}
    \text{Drift Term} \le L \sum_{t=0}^{T-1} \gamma^t(1+\err_t) \prod_{\tau=t+1}^{T-1}(1-\alpha_\tau).
\end{align*}

\textbf{3. Combining the Bounds.}
Assembling the bounds for the three components gives the inequality stated in the lemma.
\end{proof}

Before analyzing the convergence for a specific time-varying schedule, we first establish a general convergence bound that holds for any choice of sequences $\{\gamma_k\}$ and $\{\alpha_k\}$. This theorem is the stochastic analogue of the deterministic result in Theorem~\ref{thm:general_deterministic} and serves as the starting point for all subsequent rate analysis.

\begin{lemma}[General Convergence Bound]
\label{thm:general_stochastic_bound}
Let Assumptions \ref{assump:inexact_lmo_main}, \ref{assump:smoothness}, \ref{assump:stochastic_oracle}, and the norm compatibility condition hold. The sequence of iterates $\{x^k\}$ generated by Algorithm~\ref{alg:stochastic_inaxact_lmo} with time-varying parameters satisfies the following inequality for any $K \ge 1$:
\begin{align*}
    \sum_{k=0}^{K-1} \gamma_k(1-\err_k)\mathbb{E}[\|\nabla f(x^{k+1})\|_{\star}] &\le \Delta^0 + \frac{3L}{2}\sum_{k=0}^{K-1}(\gamma_k)^2(1+\err_k)^2 \\
    &\quad + 2L\sum_{k=0}^{K-1}(\gamma_k)^2(1+\err_k) + 2\sum_{k=0}^{K-1} \gamma_k \mathbb{E}[\|\nabla f(x^k) - m^{k+1}\|_{\star}],
\end{align*}
where $\Delta^0 = f(x^0) - f^*$.
\end{lemma}

\begin{proof}
The proof begins with the per-iteration guarantee from the Inexact Descent Lemma (Lemma~\ref{lemma:inexact_descent_appendix}). To obtain a guarantee on the true gradient norm, we apply the reverse triangle inequality to the main descent term, $\|m^{k+1}\|_{\star} \ge \|\nabla f(x^{k+1})\|_{\star} - \|\nabla f(x^{k+1}) - m^{k+1}\|_{\star}$. Substituting this into the descent lemma's bound gives
\begin{align*}
    f(x^{k+1}) &\le f(x^k) - \gamma_k\|\nabla f(x^{k+1})\|_{\star}(1-\err_k) + \gamma_k(1-\err_k)\|\nabla f(x^{k+1}) - m^{k+1}\|_{\star} \\
    &\quad + \gamma_k(1+\err_k)\|\nabla f(x^{k+1}) - m^{k+1}\|_{\star} + \frac{3L}{2}(\gamma_k)^2(1+\err_k)^2.
\end{align*}
Combining the two momentum error terms (whose coefficients sum to 2) and rearranging yields
\begin{align*}
    \gamma_k(1-\err_k)\|\nabla f(x^{k+1})\|_{\star} \le f(x^k) - f(x^{k+1}) + 2\gamma_k\|\nabla f(x^{k+1}) - m^{k+1}\|_{\star} + \frac{3L}{2}(\gamma_k)^2(1+\err_k)^2.
\end{align*}
We now sum this inequality from $k=0$ to $K-1$ and take the total expectation:
\begin{align*}
    \sum_{k=0}^{K-1}\gamma_k(1-\err_k)\mathbb{E}[\|\nabla f(x^{k+1})\|_{\star}] &\le \sum_{k=0}^{K-1} \mathbb{E}[f(x^k) - f(x^{k+1})] \\
    &\quad + 2\sum_{k=0}^{K-1}\gamma_k\mathbb{E}[\|\nabla f(x^{k+1}) - m^{k+1}\|_{\star}] + \sum_{k=0}^{K-1}\frac{3L}{2}(\gamma_k)^2(1+\err_k)^2.
\end{align*}
The first term on the right-hand side is a telescoping sum bounded by $\Delta^0$. The core of the proof is to decompose the sum of the momentum errors using the triangle inequality to bridge the index gap
\begin{align*}
    2\sum_{k=0}^{K-1}\gamma_k\mathbb{E}[\|\nabla f(x^{k+1}) - m^{k+1}\|_{\star}] \le 2\sum_{k=0}^{K-1}\gamma_k\left(\mathbb{E}[\|\nabla f(x^{k+1}) - \nabla f(x^k)\|_{\star}] + \mathbb{E}[\|\nabla f(x^k) - m^{k+1}\|_{\star}]\right).
\end{align*}
This decomposition separates the error into two distinct components.

\textbf{1. Bounding the Gradient Drift Component.} The first part of this sum is bounded using Assumption~\ref{assump:smoothness} ($L$-smoothness) and the property of the update step, $\|x^{k+1}-x^k\| = \gamma_k\|\hat{d}^k\| \le \gamma_k(1+\err_k)$, as
\begin{align*}
    2\sum_{k=0}^{K-1}\gamma_k\mathbb{E}[\|\nabla f(x^{k+1}) - \nabla f(x^k)\|_{\star}] \le 2\sum_{k=0}^{K-1}\gamma_k \left( L \gamma_k (1+\err_k) \right) = 2L\sum_{k=0}^{K-1}(\gamma_k)^2(1+\err_k).
\end{align*}
This term represents the accumulated error caused by the gradient changing between steps.

\textbf{2. The Momentum Lag Component.} The second part of the sum is the total accumulated momentum error, which is the final term in the theorem statement:
\begin{align*}
    2\sum_{k=0}^{K-1}\gamma_k \mathbb{E}[\|\nabla f(x^k) - m^{k+1}\|_{\star}].
\end{align*}
\textbf{Assembling the Final Bound.}
Combining all parts, we arrive at the final inequality stated in the theorem. This result cleanly separates the total progress (LHS) from the initial sub-optimality gap ($\Delta^0$), the cost from the step size's magnitude (the two $\sum(\gamma_k)^2$ terms), and the accumulated momentum lag error arising from momentum and stochastic noise.
\end{proof}

We now use the \say{General Convergence Bound} to derive an explicit convergence rate for a specific, parameter-agnostic schedule. The choice of the time-varying step size and momentum parameter is critical for ensuring that the accumulated error terms are properly controlled, leading to a guarantee of convergence. We adopt a schedule inspired by the analysis in \cite{yang2023two}, which has been proven effective for this type of recursive error structure.

\begin{theorem}[Convergence with Time-Varying Parameters]
\label{thm:time_varying_rate_final}
Let Assumptions \ref{assump:inexact_lmo_main}, \ref{assump:smoothness}, \ref{assump:stochastic_oracle}, and the norm compatibility condition hold, with a uniform inexactness bound $\err_k \le \err < 1$. If we choose the time-varying step size and momentum parameters for $k \ge 0$ as
\begin{align*}
    \gamma_k = \frac{\gamma}{(k+1)^{3/4}} \quad \text{and} \quad \alpha_k = \frac{\alpha}{(k+1)^{1/2}},
\end{align*}
for some user-chosen constants $\gamma > 0$ and $\alpha \in (0, 1]$, then after $K$ iterations, the minimum expected gradient norm is bounded as
\begin{align*}
    \min_{0 \le j < K} \mathbb{E}[\|\nabla f(x^{j+1})\|_{\star}] &= \mathcal{O}\left( \frac{\Delta^0 + L\gamma^2((1+\err)^2 + (1+\err))}{(1-\err)\gamma K^{1/4}} \right) \\
    &\quad + \mathcal{O}\left( \frac{\left( \rho\sigma\sqrt{\alpha} + L\gamma(1+\err)\alpha^{-1} \right) \log K}{(1-\err) K^{1/4}} \right).
\end{align*}
\end{theorem}

\begin{proof}
The proof starts from the General Convergence Bound (Theorem~\ref{thm:general_stochastic_bound}). After rearranging, it provides a bound on the minimum expected gradient norm:
\begin{align*}
    \left(\min_{0 \le j < K} \mathbb{E}[\|\nabla f(x^{j+1})\|_{\star}] \right) \sum_{k=0}^{K-1} \gamma_k(1-\err_k) &\le \Delta^0 + \frac{3L}{2}\sum_{k=0}^{K-1}(\gamma_k)^2(1+\err_k)^2 \\
    &\quad + 2L\sum_{k=0}^{K-1}(\gamma_k)^2(1+\err_k) + 2\sum_{k=0}^{K-1} \gamma_k \mathbb{E}[\|\nabla f(x^k) - m^{k+1}\|_{\star}].
\end{align*}
We analyze each term in this expression under the chosen parameter schedule and the uniform bound $\err_k \le \err$.

\textbf{1. Bounding the Sums of Step Sizes.}
With the schedule $\gamma_k = \gamma/(k+1)^{3/4}$, the sums have the following behavior. The progress-weighting sum on the LHS is bounded below by an integral for $K \ge 1$:
\begin{align*}
    \sum_{k=0}^{K-1} \gamma_k(1-\err_k) \ge \gamma(1-\err) \sum_{k=0}^{K-1} \frac{1}{(k+1)^{3/4}} \ge 4\gamma(1-\err)((K+1)^{1/4}-1).
\end{align*}
The squared step size sums on the RHS correspond to a convergent p-series (since the exponent $p=3/2 > 1$) and are bounded by the constant $\zeta(3/2) = \sum_{k=1}^{\infty} k^{-3/2}$ for any $K$, where $\zeta$ is the Riemann zeta function. Therefore, the two terms are bounded by $\frac{3L}{2}\gamma^2\zeta(3/2)(1+\err)^2$ and $2L\gamma^2\zeta(3/2)(1+\err)$ respectively.

\textbf{2. Bounding the Accumulated Momentum Error.}
The momentum lag component, $2\sum_{k=0}^{K-1}\gamma_k \mathbb{E}[\|\nabla f(x^k) - m^{k+1}\|_{\star}]$, is bounded by applying the results of Lemma~\ref{lemma:aligned_momentum_error}. The structure of the bound in that lemma is of the form addressed by Lemma 3 in \cite{yang2023two}. Applying their result to the two sum-of-products terms in our bound, with our chosen schedules, yields
\begin{align*}
    \sum_{k=0}^{K-1}\gamma_k \mathbb{E}[\|\nabla f(x^k) - m^{k+1}\|_{\star}] \le \left( C_1 \rho\sigma\gamma\sqrt{\alpha} + C_2 L (1+\err) \frac{\gamma^2}{\alpha} \right) \log(K+1),
\end{align*}
where $C_1, C_2$ are universal numerical constants. This shows that the total accumulated error from momentum and noise grows only logarithmically with the number of iterations.

\textbf{3. Assembling the Final Bound.}
Substituting these explicit bounds back into the main inequality:
\begin{align*}
    \left(\min_{j<K}\mathbb{E}[\|\nabla f(x^{j+1})\|_{\star}]\right) \cdot 4\gamma(1-\err)(K^{1/4}-1) &\le \Delta^0 + L\gamma^2\zeta(3/2)\left(\frac{3}{2}(1+\err)^2 + 2(1+\err)\right) \\
    &\quad + 2\left( C_1 \rho\sigma\gamma\sqrt{\alpha} + C_2 L(1+\err)\frac{\gamma^2}{\alpha} \right) \log(K+1).
\end{align*}
Dividing by $4\gamma(1-\err)(K^{1/4}-1)$ gives the final expression for the convergence rate. For large $K$, the term $(K+1)^{1/4}-1 \approx K^{1/4}$. This leads to the final bound stated in the theorem, where the dependence on all parameters, including the user-chosen schedule constants $\gamma$ and $\alpha$, is made explicit.
\end{proof}

\section{Extensions}

\subsection{$(L^0, L^1)$-smoothness}\label{app:l0l1_smooth}

We now extend our analysis from the standard $L$-smooth setting to the more general and realistic $(L^0, L^1)$-smoothness model, as defined in \cite{riabinin2025gluon}.

\begin{assumption}[Local $(L^0, L^1)$-Smoothness]
\label{assump:l0l1_smoothness}
The function $f$ is continuously differentiable and is bounded below by $f^* > -\infty$. We assume that $f$ is locally $(L^0, L^1)$-smooth along the optimization trajectory. Specifically, for any pair of consecutive iterates $x^k$ and $x^{k+1}$ generated by the algorithm, the following inequality holds:
\begin{align*}
    f(x^{k+1}) \le f(x^k) + \inp{\nabla f(x^k)}{x^{k+1}-x^k} + \frac{L^0 + L^1\|\nabla f(x^k)\|_{\star}}{2}\|x^{k+1}-x^k\|^2.
\end{align*}
\end{assumption}
This local version is sufficient for the analysis of descent methods where the step sizes are bounded, as is the case here. For clarity, we do not explicitly track the radius of the local region, as it is implicitly defined by the maximum possible step size.

\begin{theorem}[Iteration Complexity for $(L^0, L^1)$-Smoothness]
Let the function $f$ satisfy the local $(L^0, L^1)$-smoothness property (Assumption~\ref{assump:l0l1_smoothness}). Consider the deterministic method with an inexact LMO satisfying a uniform error bound $\err_k \le \err < 1$ (Assumption~\ref{assump:inexact_lmo_main}), and using the adaptive step size
\begin{align*}
    \gamma_k = \frac{\|\nabla f(x^k)\|_{\star}(1-\err_k)}{(L^0 + L^1\|\nabla f(x^k)\|_{\star})(1+\err_k)^2}.
\end{align*}
To guarantee finding an iterate $x^k$ such that $\|\nabla f(x^k)\|_{\star} \le \varepsilon$ after at most $K$ iterations, it is sufficient to run the algorithm for a number of iterations $K$ satisfying
\begin{align*}
    K \ge \frac{2\Delta^0(1+\err)^2}{(1-\err)^2} \left( \frac{L^0}{\varepsilon^2} + \frac{L^1}{\varepsilon} \right),
\end{align*}
where $\Delta^0 = f(x^0) - f^*$.
\end{theorem}

\begin{proof}
The proof derives the iteration complexity for the deterministic method under $(L^0, L^1)$-smoothness by using a step size that is chosen optimally at each iteration.

\textbf{1. Per-Iteration Descent Inequality.}
Our analysis begins with the generalized smoothness inequality for $(L^0, L^1)$-smooth functions, which holds for any step $x^{k+1} = x^k + \gamma_k \hat{d}^k$:
\begin{align*}
    f(x^{k+1}) \le f(x^k) + \inp{\nabla f(x^k)}{x^{k+1}-x^k} + \frac{L^0 + L^1\|\nabla f(x^k)\|_{\star}}{2}\|x^{k+1}-x^k\|^2.
\end{align*}
As established in the proof of Theorem~\ref{thm:general_deterministic}, we can bound the terms involving the inexact direction $\hat{d}^k$ as
\begin{itemize}
    \item $\|\gamma_k \hat{d}^k\|^2 \le (\gamma_k)^2 (1+\err_k)^2$.
    \item $\inp{\nabla f(x^k)}{\hat{d}^k} \le -\|\nabla f(x^k)\|_{\star}(1-\err_k)$.
\end{itemize}
Substituting these into the smoothness inequality gives the per-iteration guarantee
\begin{align} \label{eq:l0l1_per_iter_descent}
    f(x^{k+1}) \le f(x^k) - \gamma_k \|\nabla f(x^k)\|_{\star}(1-\err_k) + \frac{L^0 + L^1\|\nabla f(x^k)\|_{\star}}{2} (\gamma_k)^2 (1+\err_k)^2.
\end{align}

\textbf{2. Optimal Step Size and Resulting Descent.}
The right-hand side of inequality \eqref{eq:l0l1_per_iter_descent} is a quadratic function of the step size $\gamma_k$. To maximize the guaranteed descent at each step, we choose the $\gamma_k$ that minimizes this expression. The optimal choice is given by
\begin{align*}
    \gamma_k = \frac{\|\nabla f(x^k)\|_{\star}(1-\err_k)}{(L^0 + L^1\|\nabla f(x^k)\|_{\star})(1+\err_k)^2}.
\end{align*}
We now substitute this optimal step size back into the descent inequality \eqref{eq:l0l1_per_iter_descent}:
\begin{align*}
    f(x^{k+1}) \le  f(x^k) - \frac{(\|\nabla f(x^k)\|_{\star}(1-\err_k))^2}{2(L^0 + L^1\|\nabla f(x^k)\|_{\star})(1+\err_k)^2}.
\end{align*}

\textbf{3. Global Bound and Iteration Complexity.}
Summing the final descent inequality from $k=0$ to $K-1$ gives the global progress:
\begin{align*}
    \sum_{k=0}^{K-1} \frac{(\|\nabla f(x^k)\|_{\star}(1-\err_k))^2}{2(L^0 + L^1\|\nabla f(x^k)\|_{\star})(1+\err_k)^2} \le \sum_{k=0}^{K-1} (f(x^k) - f(x^{k+1})) \le \Delta^0.
\end{align*}
To derive a meaningful convergence bound from this sum, we define a progress function $\phi(t)$ that captures the contribution for a gradient norm of magnitude $t = \|\nabla f(x^k)\|_{\star}$:
\begin{align*}
    \phi(t) \eqdef \frac{t^2}{L^0 + L^1 t}.
\end{align*}
This function is monotonically increasing for $t \ge 0$. Assuming a uniform inexactness bound $\err_k \le \err < 1$, the summed inequality can be rewritten as:
\begin{align*}
    \frac{(1-\err)^2}{2(1+\err)^2} \sum_{k=0}^{K-1} \phi(\|\nabla f(x^k)\|_{\star}) \le \Delta^0.
\end{align*}
Let $g_{\min} = \min_{0 \le j < K} \|\nabla f(x^j)\|_{\star}$. Since $\phi(t)$ is increasing, we can lower-bound the sum by $K \phi(g_{\min})$:
\begin{align*}
    \frac{K (1-\err)^2}{2(1+\err)^2} \phi(g_{\min}) \le \Delta^0 \implies \phi(g_{\min}) \le \frac{2\Delta^0(1+\err)^2}{K(1-\err)^2}.
\end{align*}
To derive the iteration complexity, we require that the algorithm finds an iterate with gradient norm at most $\varepsilon$, i.e., $g_{\min} \le \varepsilon$. This is guaranteed if $\phi(\varepsilon)$ satisfies the above bound:
\begin{align*}
    \phi(\varepsilon) \le \frac{2\Delta^0(1+\err)^2}{K(1-\err)^2}.
\end{align*}
Solving for $K$ gives the required number of iterations:
\begin{align*}
    K \ge \frac{2\Delta^0(1+\err)^2}{\phi(\varepsilon)(1-\err)^2}.
\end{align*}
Substituting the definition of $\phi(\varepsilon) = \frac{\varepsilon^2}{L^0 + L^1\varepsilon}$ yields the final complexity stated in the theorem.
\end{proof}

\subsection{Layer-wise analysis}  \label{app:layerwise}

We now extend the analysis to the practical, multi-layered setting considered in \cite{riabinin2025gluon}. This framework is essential for modern deep learning, where models are composed of distinct blocks or layers with heterogeneous geometries. We assume that the parameter vector $X$ is a concatenation of $p$ blocks, $X = [X_1, \dots, X_p]$, where each layer $X_i$ belongs to its own vector space $\mathcal{X}_i$ equipped with an inner product $\langle \cdot, \cdot\rangle_{(i)}$ and norm $\|\cdot\|_{(i)}$. Each layer is updated independently via its own LMO.

This framework allows us to model a realistic scenario where not only the smoothness but also the LMO inexactness varies across layers. For example, the LMO for a simple bias vector might be solved almost exactly ($\err_i$ is small), while the LMO for a large transformer attention block might have a larger error ($\err_i$ is large).

\begin{assumption}[Layer-Wise Smoothness]
\label{assump:layerwise_smoothness}
The function $f$ is smooth with respect to each layer $i$ with a constant $L_i$. Specifically, for any iterates $X^k$ and $X^{k+1}$, the following holds:
\begin{align*}
    f(X^{k+1}) \le f(X^k) + \sum_{i=1}^p \left[ \inp{\nabla_i f(X^k)}{X_i^{k+1}-X_i^k}_{(i)} + \frac{L_i}{2}\|X_i^{k+1}-X_i^k\|^2_{(i)} \right].
\end{align*}
\end{assumption}

We first derive a general bound that holds for any choice of layer-specific, time-varying step sizes $\gamma_{k,i}$ and inexactness levels $\err_{k,i}$.

\begin{theorem}[General Convergence Bound for Layer-Wise Method]
Let Assumption~\ref{assump:layerwise_smoothness} hold. Let the layer-wise inexact LMO for each layer $i$ satisfy Assumption~\ref{assump:inexact_lmo_main} with error $\err_{k,i} < 1$. Then, for any sequence of layer-specific step sizes $\gamma_{k,i} > 0$, the following bound holds after $K$ iterations:
\begin{align*}
    \sum_{k=0}^{K-1} \sum_{i=1}^p \gamma_{k,i} \|\nabla_i f(X^k)\|_{(i)\star}(1-\err_{k,i}) \le \Delta^0 + \frac{1}{2} \sum_{k=0}^{K-1} \sum_{i=1}^p L_i (\gamma_{k,i})^2 (1+\err_{k,i})^2.
\end{align*}
\end{theorem}

\begin{proof}
The proof extends the single-layer analysis to the multi-layer setting. We start with the layer-wise smoothness inequality from Assumption~\ref{assump:layerwise_smoothness}. The update for each layer is $X_i^{k+1}-X_i^k = \gamma_{k,i} \hat{d}_i^k$. For each layer $i$ inside the summation, we can bound the terms involving its update direction $\hat{d}_i^k$ exactly as in the single-layer case:
\begin{itemize}
    \item $\|\gamma_{k,i} \hat{d}_i^k\|^2_{(i)} \le (\gamma_{k,i})^2 (1+\err_{k,i})^2$.
    \item $\inp{\nabla_i f(X^k)}{\hat{d}_i^k}_{(i)} \le -\|\nabla_i f(X^k)\|_{(i)\star}(1-\err_{k,i})$.
\end{itemize}
Substituting these bounds into the main inequality gives the per-iteration descent guarantee for the layer-wise setting:
\begin{align*}
    f(X^{k+1}) \le f(X^k) + \sum_{i=1}^p \left[ -\gamma_{k,i} \|\nabla_i f(X^k)\|_{(i)\star}(1-\err_{k,i}) + \frac{L_i}{2} (\gamma_{k,i})^2 (1+\err_{k,i})^2 \right].
\end{align*}
Rearranging to isolate the measure of progress (the sum of gradient norms) yields
\begin{align*}
    \sum_{i=1}^p \gamma_{k,i} \|\nabla_i f(X^k)\|_{(i)\star}(1-\err_{k,i}) \le f(X^k) - f(X^{k+1}) + \sum_{i=1}^p \frac{L_i}{2} (\gamma_{k,i})^2 (1+\err_{k,i})^2.
\end{align*}
Summing this from $k=0$ to $K-1$ and using the telescoping property of the function values on the right-hand side, $\sum_{k=0}^{K-1}(f(X^k) - f(X^{k+1})) = f(X^0) - f(X^K) \le \Delta^0$, gives the final result stated in the theorem.
\end{proof}

We now analyze the convergence of the layer-wise method when using an adaptive step size for each layer, chosen at each iteration to maximize the guaranteed descent. This approach provides a theoretically-grounded, adaptive learning rate policy for each layer.

\begin{theorem}[Iteration Complexity with Layer-Wise Adaptive Step Sizes]
Let Assumption~\ref{assump:layerwise_smoothness} hold. Consider the deterministic layer-wise method where the inexact LMO for each layer $i$ satisfies a uniform error bound $\err_{k,i} \le \err_i < 1$. If at each iteration $k$, the step size for each layer $i$ is chosen adaptively as
\begin{align*}
    \gamma_{k,i}^* = \frac{\|\nabla_i f(X^k)\|_{(i)\star}(1-\err_{k,i})}{L_i(1+\err_{k,i})^2},
\end{align*}
then to guarantee finding an iterate $X^k$ such that the weighted squared gradient norm $\sum_{i=1}^p \frac{\|\nabla_i f(X^k)\|_{(i)\star}^2}{L_i} \le \varepsilon^2$ after at most $K$ iterations, it is sufficient to run the algorithm for a number of iterations $K$ satisfying
\begin{align*}
    K \ge \frac{2\Delta^0}{\varepsilon^2} \max_{1 \le j \le p} \left\{ \frac{(1+\err_j)^2}{(1-\err_j)^2} \right\}.
\end{align*}
\end{theorem}

\begin{proof}
The proof begins with the per-iteration descent inequality derived in the proof of the General Convergence Bound theorem:
\begin{align*}
    f(X^{k+1}) \le f(X^k) + \sum_{i=1}^p \left[ -\gamma_{k,i} \|\nabla_i f(X^k)\|_{(i)\star}(1-\err_{k,i}) + \frac{L_i}{2} (\gamma_{k,i})^2 (1+\err_{k,i})^2 \right].
\end{align*}
The adaptive step size $\gamma_{k,i}^*$ is the value that minimizes the term in the square brackets for each layer $i$. Substituting this optimal choice into the inequality yields the maximum guaranteed descent at each step
\begin{align*}
    f(X^{k+1}) \le f(X^k) - \sum_{i=1}^p \frac{(\|\nabla_i f(X^k)\|_{(i)\star}(1-\err_{k,i}))^2}{2 L_i (1+\err_{k,i})^2}.
\end{align*}
Summing this from $k=0$ to $K-1$ gives the global progress bound
\begin{align*}
    \sum_{k=0}^{K-1} \sum_{i=1}^p \frac{\|\nabla_i f(X^k)\|_{(i)\star}^2 (1-\err_{k,i})^2}{2 L_i (1+\err_{k,i})^2} \le \Delta^0.
\end{align*}
We assume a uniform (but potentially layer-specific) inexactness bound $\err_{k,i} \le \err_i < 1$. The inequality becomes
\begin{align*}
    \frac{1}{2} \sum_{k=0}^{K-1} \sum_{i=1}^p \frac{(1-\err_i)^2}{(1+\err_i)^2} \frac{\|\nabla_i f(X^k)\|_{(i)\star}^2}{L_i} \le \Delta^0.
\end{align*}
Let $w_i \eqdef \frac{(1-\err_i)^2}{(1+\err_i)^2}$ be the layer-specific error weight and let $G_k^2 \eqdef \sum_{i=1}^p \frac{\|\nabla_i f(X^k)\|_{(i)\star}^2}{L_i}$ be the weighted squared gradient norm. The bound can be written as $\frac{1}{2} \sum_{k=0}^{K-1} \sum_{i=1}^p w_i \frac{\|\nabla_i f(X^k)\|_{(i)\star}^2}{L_i}$. To derive a clear complexity, we can lower-bound the weights by their minimum value, $w_{\min} = \min_{1 \le j \le p} w_j$:
\begin{align*}
    \sum_{i=1}^p w_i \frac{\|\nabla_i f(X^k)\|_{(i)\star}^2}{L_i} \ge w_{\min} \sum_{i=1}^p \frac{\|\nabla_i f(X^k)\|_{(i)\star}^2}{L_i} = w_{\min} G_k^2.
\end{align*}
Substituting this into the main inequality gives
    $\frac{w_{\min}}{2} \sum_{k=0}^{K-1} G_k^2 \le \Delta^0.$
Let $G_{\min}^2 = \min_{0 \le j < K} G_j^2$. Since $\sum_{k=0}^{K-1} G_k^2 \ge K \cdot G_{\min}^2$, we have
\begin{align*}
    \frac{K w_{\min}}{2} G_{\min}^2 \le \Delta^0 \implies G_{\min}^2 \le \frac{2\Delta^0}{K w_{\min}}.
\end{align*}
To guarantee that the weighted squared gradient norm $G_{\min}^2 \le \varepsilon^2$, we require
\begin{align*}
    \frac{2\Delta^0}{K w_{\min}} \le 
    \varepsilon^2 \implies K \ge \frac{2\Delta^0}{\varepsilon^2 w_{\min}}.
\end{align*}
Substituting back the definition of $w_{\min} = \min_{j} \frac{(1-\err_j)^2}{(1+\err_j)^2}$ gives the final complexity
\[
    K \ge \frac{2\Delta^0}{\varepsilon^2} \frac{1}{\min_{j} \frac{(1-\err_j)^2}{(1+\err_j)^2}} = \frac{2\Delta^0}{\varepsilon^2} \max_{j} \frac{(1+\err_j)^2}{(1-\err_j)^2}. \qedhere
\]
\end{proof}

\paragraph{Discussion}
This result for the layer-wise adaptive method is powerful as it provides a concrete, theoretically-grounded learning rate policy for each layer that leads to an optimal convergence rate.
\begin{itemize}
    \item \textbf{Optimal Rate:} The resulting complexity is $\mathcal{O}(1/\varepsilon^2)$, which is the optimal rate for deterministic non-convex optimization. This confirms that the layer-wise adaptive strategy is efficient.
    \item \textbf{Layer-Wise Adaptivity:} The step size $\gamma_{k,i}^*$ for each layer adapts to its own local properties: its current gradient norm $\|\nabla_i f(X^k)\|_{(i)\star}$, its smoothness $L_i$, and its LMO precision $\err_{k,i}$. This is a practical and intuitive update rule that naturally assigns smaller steps to sharper or less precisely updated layers.
    \item \textbf{Bottleneck Effect of Inexactness:} The final complexity is degraded by a ``worst-layer" factor, $\max_{j} \frac{(1+\err_j)^2}{(1-\err_j)^2}$. This shows that the global convergence speed is bottlenecked by the single layer with the worst combination of update quality and infeasibility cost. A single layer with very high inexactness ($\err_j \to 1$) will cause this factor to explode, dominating the iteration complexity regardless of how precisely the other layers are handled.
\end{itemize}

\newpage

\section{Additional Experiments and Details}

This section provides supplementary details and results to support the empirical findings presented in Section~\ref{sec:experiments} of the main paper.

\subsection{Experimental setup: details}
\label{app:exp_details}

We provide additional details regarding the experimental setups for both the nanoGPT and CIFAR-10 benchmarks. Further details on model architectures and data processing can be found in the respective code repositories cited in the main text. All experiments were conducted on a system equipped with 4 $\times$ NVIDIA A100 GPUs.

\paragraph{nanoGPT on FineWeb.}
The hyperparameter grids used for the sweeps shown in Figure~\ref{fig:nanoGPT_heatmaps} are visible on the axes of the plots themselves. The LMO inexactness $\err$ is controlled by the number of iterations of the \approxalg{PolarExpress} algorithm; fewer iterations correspond to a higher error $\err$.

\paragraph{CNN on CIFAR-10.}
For the results presented in Figure~\ref{fig:cifar_variability}, we performed a comprehensive grid search over the following hyperparameters for each level of LMO precision (i.e., for each number of \approxalg{Newton-Schulz} iterations):
\begin{itemize}
    \item \textbf{Step Size ($\gamma$):} A logarithmically spaced grid from $2^{-12}$ to $2^{-5}$.
    \item \textbf{Momentum ($\alpha$):} A linearly spaced grid from $0.05$ to $0.90$ in increments of $0.05$.
\end{itemize}
The ``Best Accuracy" (solid lines) in Figure~\ref{fig:cifar_variability} corresponds to the maximum test accuracy achieved at a given step size by tuning over the momentum values. The ``Worst Accuracy" (dashed lines) corresponds to the minimum test accuracy from the same sweep.

\subsection{Additional results for nanoGPT}
\label{app:additional_nanogpt}

\begin{figure}[h!]
    \centering
    \begin{subfigure}[t]{0.49\linewidth}
        \centering
        \includegraphics[width=\linewidth]{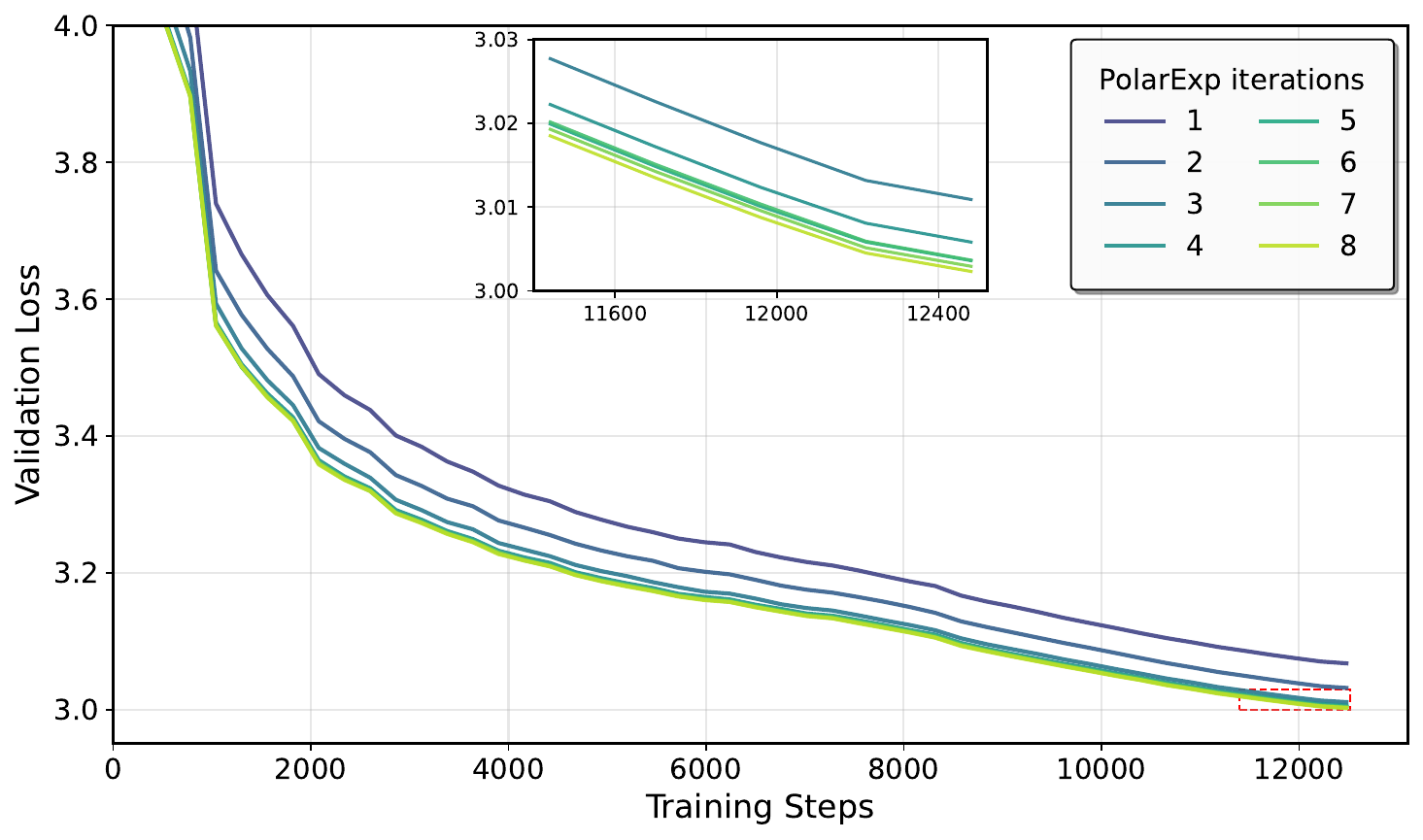}
        \caption{Improvements saturate after $\sim$5 \approxalg{PolarExp} iterations.}
    \label{fig:appendix_nanogpt_convergence_full}
    \end{subfigure}
    \hfill
    \begin{subfigure}[t]{0.49\linewidth}
        \centering
        \includegraphics[width=\linewidth]{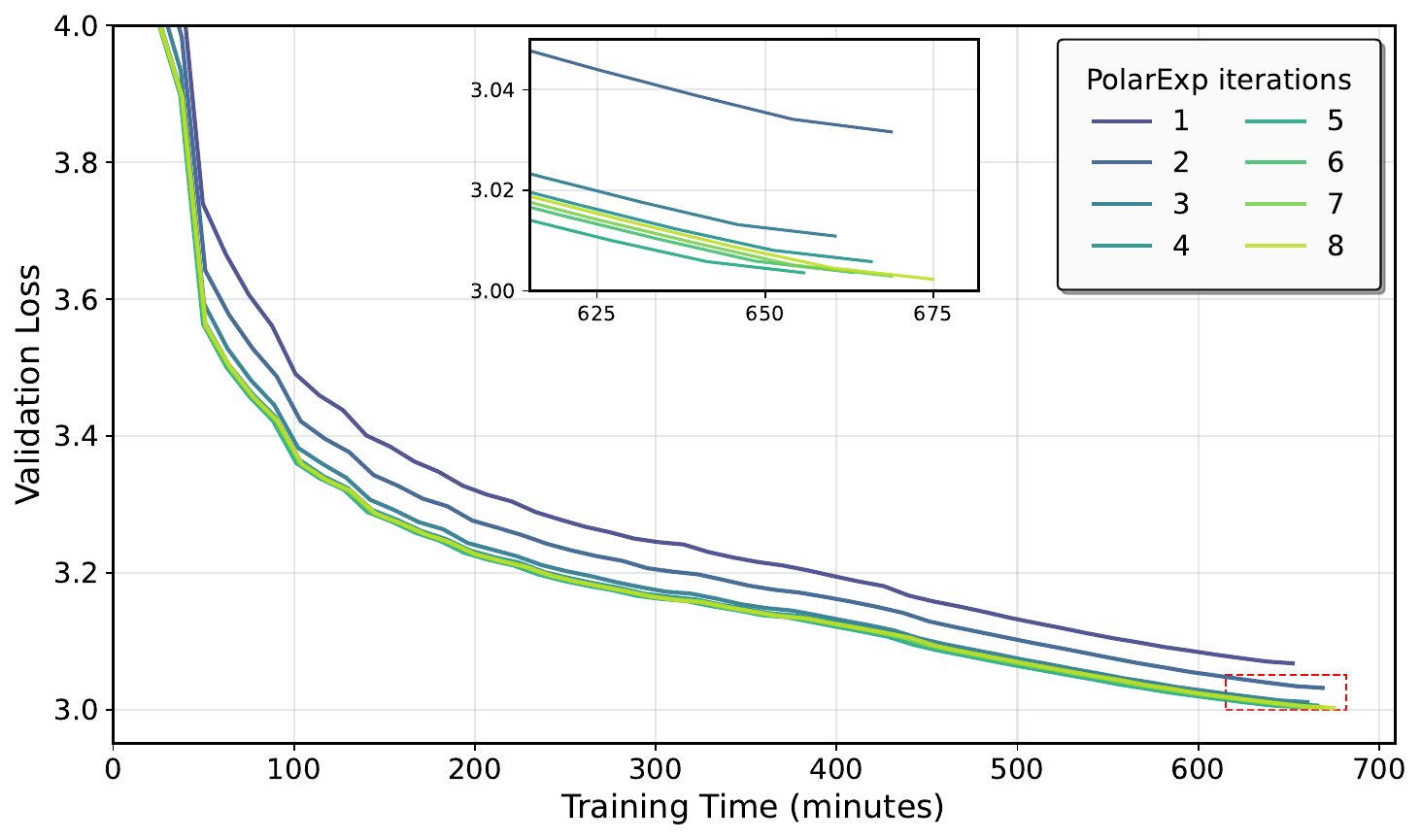}
        \caption{Training time grows slightly with the iteration count.}
        \label{fig:appendix_nanogpt_time}
    \end{subfigure}
    \caption{Supplementary results for nanoGPT on FineWeb: convergence behavior (left) and runtime scaling (right) as the number of \approxalg{PolarExpress} (PolarExp) iterations varies.}
    \label{fig:appendix_nanogpt_combined}
\end{figure}

\paragraph{Performance Degradation with Increasing Precision.}
As mentioned in the main text, increasing the LMO precision yields diminishing returns. Figure~\ref{fig:appendix_nanogpt_convergence_full} provides the full validation loss curves for up to 8 iterations of the \approxalg{PolarExpress} algorithm. The plot clearly shows that while the most significant performance gains are achieved when moving from 1 to 3 iterations, further increases in precision (from 3 to 5, and 5 to 8) continue to improve the convergence speed and final validation loss, albeit with smaller margins. This observation confirms the trend of diminishing returns and empirically validates the degradation factor present in our theoretical bounds.

Table~\ref{tab:nanogpt_val_loss} complements these curves by reporting the final validation loss for each setting, highlighting the diminishing returns observed beyond five \approxalg{PolarExpress} iterations.

\begin{table}[h]
\centering
    \caption{Final validation loss and total training time for NanoGPT on FineWeb across different numbers of \approxalg{PolarExpress} iterations used for the LMO approximation.}
    \label{tab:nanogpt_val_loss}
    \begin{tabular}{@{}lcccccccc@{}}
        \toprule
        \approxalg{PolarExpress} iterations & 1 & 2 & 3 & 4 & 5 & 6 & 7 & 8 \\
        \midrule
        Validation loss & 3.0675 & 3.0316 & 3.0109 & 3.0058 & 3.0036 & 3.0037 & 3.0029 & 3.0023 \\
        Training time (minutes) & 652.2 & 668.8 & 660.4 & 665.8 & 655.7 & 663.1 & 668.8 & 675.0 \\
        \bottomrule
\end{tabular}
\end{table}

\paragraph{LMO Error Analysis.}
Figure~\ref{fig:lmo_error} shows the LMO error (summed over layers) after $p$ steps of \approxalg{PolarExpress} iterations throughout training, calculated with respect to the ``true" (exact up to machine precision) LMO obtained by running \approxalg{PolarExpress} for 100 iterations. The plot demonstrates that higher iteration counts consistently achieve lower LMO error levels, with $p=8$ maintaining final error of $\sim$1.7 compared to $p=1$ reaching $\sim$23.6. The error remains mostly close to constant (except rare spikes) across training for every $p$, albeit with an initial increase at the beginning (to $\sim$26.9 for $p=1$, and to $\sim$3.5 for $p=8$). These results provide an explanation for the saturation of performance beyond certain LMO quality, as the theoretical benefits of higher precision diminish once the LMO error reaches sufficiently low levels.

\begin{figure}[h!]
    \centering
    \includegraphics[width=0.6\linewidth]{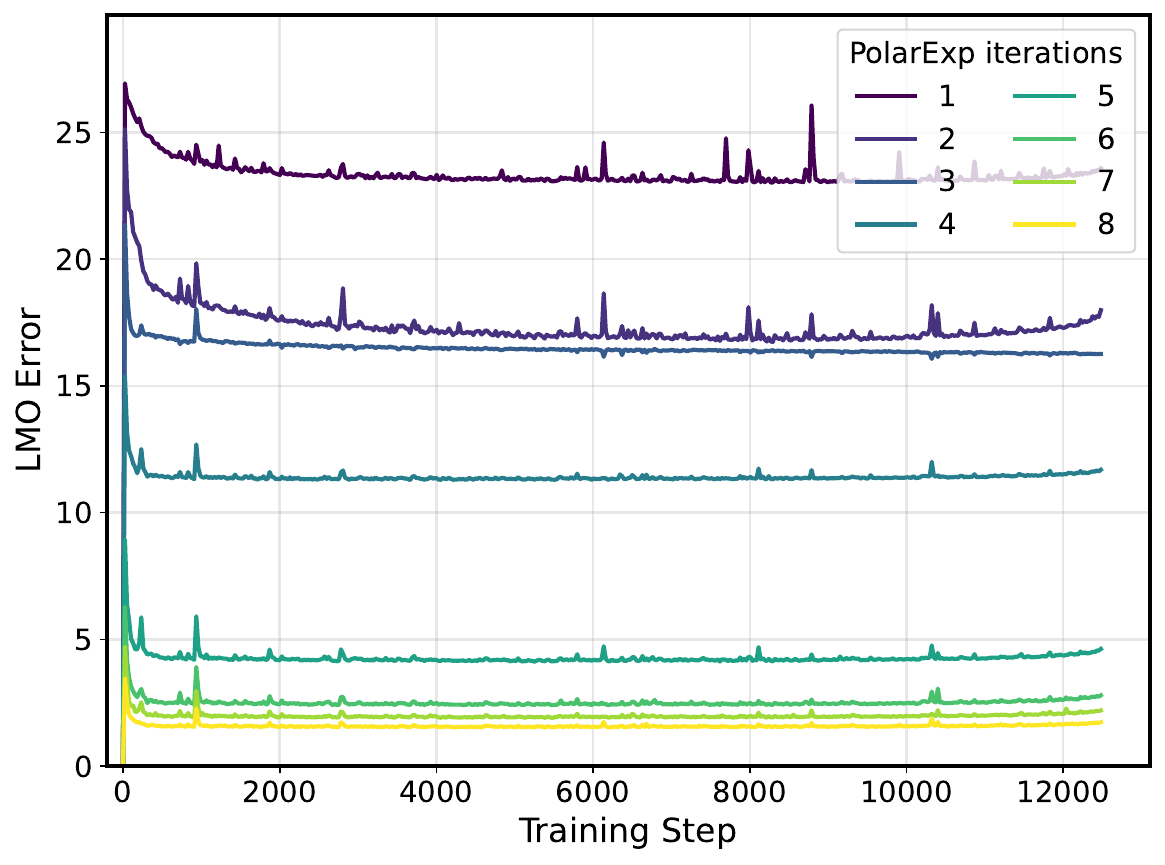}
    \caption{LMO error evolution across training steps for different numbers of \approxalg{PolarExpress}  (PolarExp) iterations. Higher iteration counts maintain lower orthogonalization error throughout training.}
    \label{fig:lmo_error}
\end{figure}

\subsection{Additional results for CIFAR-10}

\paragraph{Full Hyperparameter Grids.}
To complement the summary plots in the main text, Figure~\ref{fig:appendix_cifar_heatmaps} presents the full heatmap of final test accuracies across the entire hyperparameter grid for different levels of LMO precision. These plots provide a more detailed view of the optimizer's stability. For a highly precise LMO (e.g., 8 \approxalg{Newton-Schulz} iterations), a large contiguous region of high performance is visible, indicating robustness to hyperparameter choices. In contrast, for a highly inexact LMO (e.g., 1 \approxalg{Newton-Schulz} iteration), the high-performance region is smaller and more fragmented, highlighting the optimizer's sensitivity to tuning when the LMO is imprecise.

\begin{figure*}[h!]
\centering
\includegraphics[width=\linewidth]{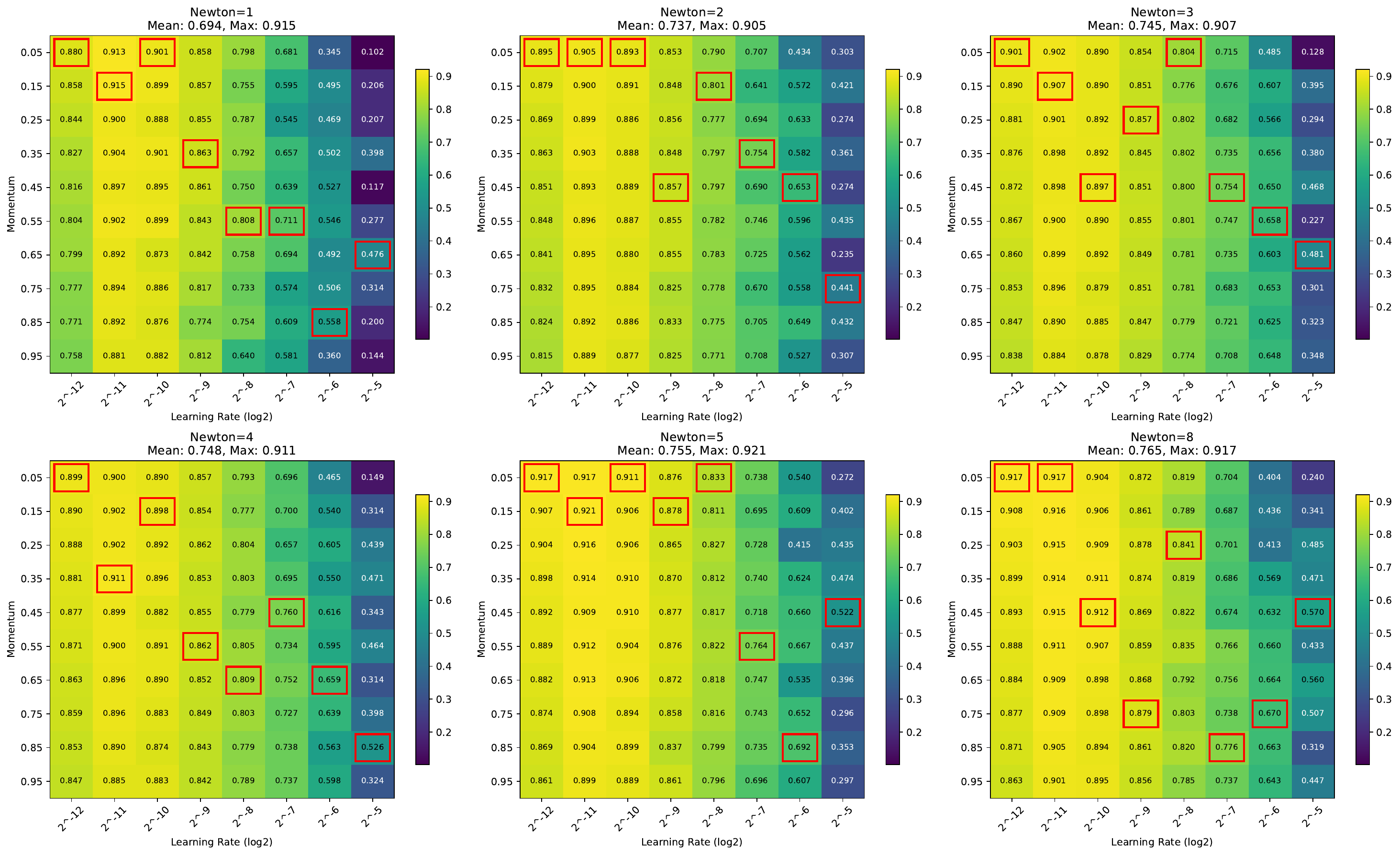}
    \caption{Full final test accuracy heatmaps on CIFAR-10 across a grid of step sizes and momentum values for different numbers of \approxalg{Newton-Schulz} (``Newton'') iterations. Brighter colors indicate higher final test accuracy. The region of high performance is visibly larger and more stable for higher numbers of \approxalg{Newton-Schulz} iterations.}
    \label{fig:appendix_cifar_heatmaps}
\end{figure*}

\end{document}